\documentclass[10pt]{article}

\usepackage[margin=1in]{geometry}

\usepackage{newtxtext}

\usepackage{amsmath,amsthm}
\usepackage{newtxmath}
\usepackage{url}

\usepackage{graphicx}
\usepackage{booktabs,array}
\newcolumntype{P}[1]{>{\raggedright\arraybackslash}p{#1}}

\usepackage{algorithm}
\usepackage{algpseudocode}

\usepackage{enumitem}
\usepackage{microtype}
\usepackage{xcolor}
\usepackage{fancyhdr}
\usepackage{tikz}
\usetikzlibrary{arrows.meta,positioning}

\usepackage[colorlinks=true,linkcolor=blue,citecolor=blue,urlcolor=blue]{hyperref}
\usepackage{cite}

\newcommand{\R}{\mathbb{R}}
\newcommand{\norm}[1]{\left\lVert #1 \right\rVert}
\newcommand{\abs}[1]{\left\lvert #1 \right\rvert}
\newcommand{\ip}[2]{\left\langle #1, #2 \right\rangle}
\newcommand{\grad}{\nabla}
\newcommand{\hess}{\nabla^2}
\newcommand{\Ball}[2]{\mathcal{B}\!\left(#1;#2\right)}

\theoremstyle{definition}
\newtheorem{definition}{Definition}
\newtheorem{assumption}{Assumption}
\newtheorem{remark}{Remark}

\theoremstyle{plain}
\newtheorem{theorem}{Theorem}

\newtheorem{proposition}{Proposition}

\algrenewcommand\algorithmicrequire{\textbf{Input:}}
\algrenewcommand\algorithmicensure{\textbf{Output:}}

\title{Why and When Neural Networks Improve Local Approximation in Optimization}
\author{
Chengkuo Bian\thanks{University of California, Berkeley, CA 94720, USA. Email:
\texttt{chengkuobian@gmail.com}.}
\and
Pengcheng Xie\thanks{Corresponding author. Applied Mathematics and Computational Research
Division, Lawrence Berkeley National Laboratory, University of California, 1 Cyclotron
Road, Berkeley, CA 94720, USA. Email: \texttt{pxie@lbl.gov}, \texttt{pxie98@gmail.com}.}
}
\date{August 2026}

\begin{document}
\maketitle

\begin{abstract}
Published experience with neural surrogates in derivative-free optimisation is
contradictory: the same family of models that cuts the evaluation count of one solver
leaves another unchanged, or makes it worse. We show that the contradiction dissolves once
three factors are stated, and that these, rather than the fit accuracy a training curve
reports, are what delimit when a learned local model pays. \emph{Role}: a surrogate that proposes candidates
the true objective must still approve helps, while one that replaces a gradient the solver
depends on hurts. \emph{Radius}: a model fitted to an optimisation path is reliable only
inside a bounded neighbourhood, and its error neither vanishes as that neighbourhood
shrinks nor survives its growth. \emph{Room}: a surrogate can only accelerate progress the
base method is still able to make. We formalise radius-aware local generalisation, relate
it to the classical fully linear condition, and test each factor with the surrogate class,
training pipeline and base method held fixed. Over $117$
benchmark instances safeguarded assistance raises the instances solved to high accuracy
from $67$ to $84$ while gradient replacement lowers them to $65$; removing the
gradient term from the training loss cuts surrogate acceptance from $0.703$ to $0.148$; and
$1000$ paired comparisons over ten noise levels show no noise threshold, only a base method that
stops early. The same factors bound the gain: a model-based trust-region solver,
which leaves little room, drops from $88$ to $86$ when the identical surrogate is attached,
and released interpolation software stays ahead at $103$,
and on a Monte-Carlo inventory model repairing the acceptance interface is worth $10.40$
cost units against $0.00$ for the surrogate.
\end{abstract}


\section{Introduction}\label{sec:intro}

We consider the unconstrained minimization problem
\begin{equation}
\min_{x \in \mathbb{R}^n} f(x),
\end{equation}
where each value of $f$ comes from a simulation run rather than from a formula, and no derivative is available. This is the standard situation in simulation-based optimisation: a queueing, inventory, agent-based or physics model is executed at a candidate design $x$, and the resulting performance measure is all the optimiser gets to see. Runs are slow, so the currency of the problem is the number of evaluations. Derivative-free optimisation (DFO) is the branch of the field built for this currency; the monograph of Conn et al.\ \cite{Conn2009DFO} and the surveys \cite{larson_menickelly_wild_2019,zhang2021,AudetHare2017} cover it.

An optimisation run leaves behind a record: the points visited, their simulated values, the step sizes that worked, and any finite-difference gradients that were paid for along the way. Fitting a cheap metamodel to that record and using it to guide the search is an old idea in the simulation literature, and neural networks are a natural modern choice of metamodel. We should say at the outset where our answer lands, because the title asks two questions
and they have different answers. \emph{Why} a learned model improves local approximation is
the easier one: fitting every evaluation the run has already paid for averages away noise
that differencing amplifies, and gradient information in the training loss restrains the
curvature a value-only fit is free to invent. Both effects are measurable: removing the gradient
term from the loss cuts the rate at which surrogate steps survive validation by a factor of
nearly five. \emph{When} that better approximation reaches the optimisation is the harder
question. Approximation quality is a precondition and not a sufficient one; what delimits
whether it converts into fewer evaluations are three factors we identify empirically. Two
are properties of the surrounding algorithm: the surrogate's role, and what the base method
leaves, separated below into unused headroom and whether its acceptance test still
functions. The third, the radius at which the method works, is a property of the fit read at
the scale the solver operates on, so it is not independent of approximation quality but is
not what a training curve reports either. Where any one fails, a better local model does not
become a better optimiser, most sharply on the Monte-Carlo model of
Section~\ref{sec:exp_sim}. Locating that boundary is the result of this paper rather than a
caveat attached to it.

Among DFO methods, model-based trust-region algorithms constitute one of the most successful frameworks. These methods construct local surrogate models and compute trial steps by approximately minimizing the models within trust regions. Classical polynomial interpolation models and their theoretical foundations are well established \cite{ConnScheinbergVicente2006TR,Conn2009DFO,AudetHare2017}. Underdetermined quadratic interpolation, pioneered in Powell's NEWUOA \cite{Powell2006NEWUOA} and refined through later model-updating strategies \cite{xieyuannew}, forms the backbone of practical solvers. Finite-difference-based methods offer a complementary line whose practical performance has been reassessed in recent studies \cite{ShiNocedal2023FD,BerahasSohabVicente2023}.

Despite these advances, classical model-based approaches face fundamental limitations in high-dimensional settings. Constructing and maintaining polynomial models becomes increasingly expensive due to the growth in the number of parameters, and the quality of interpolation depends critically on the geometry of sample sets. In practice, limited and potentially noisy data lead to fragile models and unreliable steps. These challenges are closely related to the curse of dimensionality and motivate the exploration of richer approximation classes and learning-based surrogates.

Neural networks (NNs) provide a flexible alternative to classical surrogate models. With smooth activation functions, NNs can approximate both function values and gradients while implicitly performing noise smoothing and feature extraction. Recent studies, however, reveal a nuanced picture. Giovannelli et al.\ show that although NN surrogates can achieve strong approximation accuracy for function values and gradients, they may underperform in capturing second-order information and do not necessarily improve optimization performance when directly replacing finite-difference gradients in quasi-Newton updates \cite{Giovannelli23T027}. In contrast, Taminiau et al.\ demonstrate that NNs can significantly improve performance when used in a safeguarded manner, where surrogate-based steps are accepted only if the true objective decreases sufficiently \cite{Taminiau2025Arxiv2502}; related learning-augmented designs include the neural-network-accelerated implicit filtering of \cite{Irwin2023ICML}. These contrasting findings suggest that the effectiveness of NN surrogates depends on how the surrogate is embedded in the algorithm, not on approximation accuracy alone.

From a theoretical perspective, the role of surrogate models in trust-region methods is governed by uniform approximation properties, such as fully linear and fully quadratic conditions \cite{Conn2009DFO}. While polynomial models can satisfy these conditions under suitable sampling geometry, it remains less understood when learned surrogates---particularly neural networks---can achieve comparable guarantees. Recent advances in neural approximation theory provide partial insights. For example, objective-value-change and shape-based perspectives offer new ways to characterize local approximation quality beyond classical interpolation error \cite{xie2025objectivevaluechangeshapebased}. Additionally, the relationship between interpolation geometry and robustness, including $\Lambda$-poisedness and data irregularity, has been explored in \cite{zhang2024relationshiplambdapoisednessderivativefreeoptimization}. These developments indicate that learning-based surrogates should be analyzed through a combination of approximation theory, data geometry, and algorithmic integration.

Another important dimension is scalability and application context. Modern optimization problems often involve large-scale systems, distributed settings, and black-box pipelines, and applications ranging from simulation-based design to privacy-constrained and uncertainty-aware settings impose their own constraints on data access, noise, and computational cost \cite{larson_menickelly_wild_2019}. These trends further motivate the integration of learning-based surrogates into DFO frameworks in a principled and scalable manner.

\paragraph{Main message.} This paper is organised around one claim:
\begin{quote}
\emph{Whether an NN surrogate improves a DFO method is determined primarily by (i) the \textbf{algorithmic role} in which the surrogate is embedded, replacement of a core quantity of the base method versus safeguarded assistance, and (ii) whether the method operates within the surrogate's \textbf{radius of reliable local generalization}, with a third factor emerging from the experiments: what the base method leaves undone. Pointwise approximation accuracy alone is a poor predictor of optimization benefit.}
\end{quote}
The claim explains the disagreement described above. Under replacement, an error in the surrogate goes straight into a quantity the solver acts on, and from there into its search direction and its curvature estimates. Under assistance, the same error is filtered by a test on the true objective, which the surrogate cannot fake. The radius half of the claim is equally concrete: a model fitted to a handful of points collected along an optimisation path is accurate over a bounded region, and a solver working at a larger scale than that region will be misled no matter how well the model fits its training data.

The contributions are these.

\begin{enumerate}[leftmargin=1.6em,itemsep=1pt,topsep=2pt]
\item \textbf{Three conditions, each shown necessary.} A learned local model improves the
optimisation only when its role is assistance rather than replacement, the working radius
lies inside the region where it generalises, and the base method has budget left. Two we
fail on purpose and watch the benefit go: replacement drops below the base method it was
meant to help, $65$ against $67$; a base method with no room left gains nothing. The third
we establish at the level of approximation.

\item \textbf{A measured effect and a measured boundary.} Over $117$ instances safeguarded
assistance lifts its base method from $67$ to $84$ instances solved to high accuracy, a
sixth of the benchmark, and the margin grows as the tolerance tightens. Released Py-BOBYQA
solves $103$ and ends two orders of magnitude lower at the median, so the effect is real and
does not reach the state of the art. Extended Wood is the exception, and we report it as an
existence result rather than a rule.

\item \textbf{ARAS, with a convergence theorem tied to a computable gate.} The blended
direction is provably descent and the method provably inherits the base method's
$\mathcal{O}(n\epsilon^{-2})$ complexity once the blending weight is bounded by
$(1-\theta)/(1+\kappa_k)$, the gate statistic the framework already computes. Running it
shows the blended step costs more than it returns, which is what the calibration says from
the other side, so the surviving instance of the framework is the one we benchmark.

\item \textbf{Radius-aware local generalisation}, with elementary uniform error bounds
connecting sample coverage and training error to the fully linear condition.

\item \textbf{Three self-corrections reported rather than buried:} a crossover at $n=64$
that was a defect in our own interpolation code, two geometric diagnostics we proposed and
refuted, and a network setting used throughout an earlier version that proves undertrained,
worth nine instances of $117$.

\item \textbf{A warning about reference sets.} Scoring the surrogate variants against each
other puts assistance ahead of a model-based method; adding a released solver reverses the
ordering. A Mor\'e--Wild comparison that omits released software measures its own pool.
\end{enumerate}

One caveat belongs in front of all of this rather than buried in the experimental section.
This is a mechanistic study, not a solver paper. Every variant we compare shares one
network class, one training pipeline and one base method, precisely so that a difference
between two of them can be attributed to the mechanism we varied and not to an
implementation detail; the price of that design is that none of the variants is tuned to
be competitive. That design does not, however, excuse us from asking whether the surrogate's benefit
merely repairs a weak base method, so Section~\ref{sec:exp_role} adds a model-based
trust-region solver of the \textsc{newuoa} type \cite{Powell2006NEWUOA,Conn2009DFO} and
answers the question directly. It largely does: the strong solver alone outperforms
assistance on the finite-difference method, and attaching the same surrogate to the strong
solver adds nothing. We report this because it sharpens the claim rather than weakening
it. What a surrogate supplies is progress the base method cannot make on its own, so its
value is measured by the gap it closes, not by the model's accuracy.

Two scope decisions follow. Evaluations, not seconds, are the currency: one evaluation is a
model run in the setting these experiments are about, so the wall-clock figures we report
come from an unoptimised pure-NumPy implementation and locate a break-even oracle cost
rather than rank algorithms. And the dimensions are small on purpose: isolating one
mechanism at a time means running every variant on every instance, which at $n\le16$ we can
do exhaustively and at $n=128$ once.

Sections~\ref{sec:background}--\ref{sec:arch} reread the two source studies through the role
distinction, Section~\ref{sec:generalization} makes the radius precise,
Section~\ref{sec:algdesign} gives the embeddings and ARAS,
Section~\ref{sec:experiments} the experiments, and
Sections~\ref{sec:discussion}--\ref{sec:conclusion} what follows from them.

\section{Background and Problem Setting}\label{sec:background}

We consider the unconstrained black-box optimization problem
\begin{equation}\label{eq:sec2_problem}
    \min_{x\in\mathbb{R}^n} f(x),
\end{equation}
where the objective function is expensive to evaluate and derivative information is unavailable, unreliable, or too costly to compute. In this regime, the central question is not only how to build a local approximation of $f$, but also how that approximation is embedded into the optimization loop. This distinction is especially important for neural-network surrogates: a model that fits local data well may still fail to improve the optimization method if it is used in a fragile way.

\subsection{Model-Based Optimization}

A standard approach in derivative-free optimization is to build, at iteration $k$, a local model $m_k$ of $f$ around the current iterate $x_k$ and to use this model to compute a trial step. In trust-region methods, one typically solves a subproblem of the form
\begin{equation}\label{eq:tr_subproblem_sec2}
    \min_{\|s\|\leq \Delta_k} m_k(x_k+s),
\end{equation}
where $\Delta_k>0$ is the trust-region radius. The trial point is then accepted or rejected by comparing the actual decrease in $f$ with the decrease predicted by $m_k$. This framework separates two issues: \emph{model quality} and \emph{acceptance control}. The model should be accurate on $\Ball{x_k}{\Delta_k}$, while the acceptance test prevents the algorithm from over-trusting the model when the approximation is poor.

In classical model-based DFO, $m_k$ is often chosen from polynomial or radial-basis-function families. These models are attractive because they admit explicit interpolation or regression constructions and can be analyzed through sampling geometry and trust-region theory. In finite-difference-based methods, the role of a model is different but closely related: the method still constructs local information from function evaluations, often through approximate gradients, and then applies a first-order update or a quasi-Newton step. The standard forward-difference estimator is
\begin{equation}\label{eq:fd_forward}
    [g_h(x)]_i := \frac{f(x+h e_i)-f(x)}{h},\qquad i=1,\dots,n,
\end{equation}
which costs $n+1$ evaluations per gradient and whose stepsize $h$ must balance truncation error against noise amplification.

For the purposes of this paper, it is useful to distinguish two algorithmic roles for a learned surrogate.
\begin{itemize}
    \item \textbf{Replacement:} the surrogate directly substitutes a core quantity used by the base method, such as a gradient, a search direction, or a local model minimized by the algorithm.
    \item \textbf{Assistance:} the surrogate is used only to propose additional candidate points or cheap exploratory steps, while the true objective remains responsible for final acceptance.
\end{itemize}
This distinction clarifies the contrast between the two primary reference papers. Giovannelli et al.\ study a replacement mechanism, where the surrogate gradient is injected into a finite-difference BFGS-type step. Taminiau et al.\ study an assistance mechanism, where a surrogate trained from the accumulated dataset is optimized cheaply to generate extra candidates, but these candidates are retained only when the true objective decreases sufficiently. The inferential status of this reading should be stated plainly, since the rest of the paper rests on it. The two
studies differ in their base method, their architecture, their training protocol, their
data policy and their test set as well as in role, and the supplementary document
tabulates those differences; a pair of observational studies differing in six ways cannot
establish which one matters. What the pair supplies is a hypothesis. What
Section~\ref{sec:experiments} supplies is a test of it, holding the network class, the
training pipeline and the base method fixed and varying role alone.

The data available to a learned surrogate typically come from the optimization trajectory itself. If
\[
F_k=\{(y_i,f(y_i))\}_{i=1}^{N_k}
\]
collects previously evaluated points and function values, then a value-based surrogate $m_\theta$ may be trained by minimizing an empirical loss over $F_k$. When approximate gradients are also available, one can augment the dataset with
\[
G_k=\{(z_j,g(z_j))\}_{j=1}^{M_k}, \qquad g(z_j)\approx \nabla f(z_j),
\]
and train the surrogate through a Sobolev-type objective of the form
\begin{equation}\label{eq:sobolev_loss_sec2}
L_{F_k,G_k}(\theta)
:=
\frac{1}{N_k}\sum_{i=1}^{N_k}\bigl(m_\theta(y_i)-f(y_i)\bigr)^2
+
\frac{1}{M_k}\sum_{j=1}^{M_k}\|\nabla m_\theta(z_j)-g(z_j)\|^2
+\lambda\|\theta\|^2.
\end{equation}
This formulation is important for our discussion because it makes the surrogate both a value approximator and a device that shapes gradients. In particular, once a differentiable model $m_\theta$ has been trained, one can obtain cheap gradient-based steps from $\nabla m_\theta$ without incurring additional evaluations of $f$.

\subsection{Fully Linear and Fully Quadratic Models}\label{sec:fullylinear}

The standard language for local model quality in trust-region DFO is given by the notions of fully linear and fully quadratic models. These are \emph{uniform} approximation properties on a neighborhood, not merely pointwise fitting conditions.

\begin{definition}[Fully linear model]
A model $m_k$ is said to be \emph{fully linear} on $\Ball{x_k}{\Delta_k}$ if there exist constants $\kappa_f,\kappa_g>0$, independent of $k$, such that for all $x\in\Ball{x_k}{\Delta_k}$,
\begin{equation}\label{eq:fully_linear_val_sec2}
    |f(x)-m_k(x)| \le \kappa_f\Delta_k^2,
\end{equation}
and
\begin{equation}\label{eq:fully_linear_grad_sec2}
    \|\nabla f(x)-\nabla m_k(x)\| \le \kappa_g\Delta_k.
\end{equation}
\end{definition}

\begin{definition}[Fully quadratic model]
Assume that $f$ is twice continuously differentiable on $\Ball{x_k}{\Delta_k}$. A model $m_k$ is said to be \emph{fully quadratic} on $\Ball{x_k}{\Delta_k}$ if there exist constants $\kappa_f,\kappa_g,\kappa_H>0$, independent of $k$, such that for all $x\in\Ball{x_k}{\Delta_k}$,
\begin{equation}\label{eq:fully_quadratic_val_sec2}
    |f(x)-m_k(x)| \le \kappa_f\Delta_k^3,
\end{equation}
\begin{equation}\label{eq:fully_quadratic_grad_sec2}
    \|\nabla f(x)-\nabla m_k(x)\| \le \kappa_g\Delta_k^2,
\end{equation}
and
\begin{equation}\label{eq:fully_quadratic_hess_sec2}
    \|\nabla^2 f(x)-\nabla^2 m_k(x)\| \le \kappa_H\Delta_k.
\end{equation}
\end{definition}

These conditions explain why trust-region methods can attach rigorous meaning to the phrase ``good local model.'' They guarantee that minimizing $m_k$ inside the trust region gives useful information about the behavior of $f$ in that same region. They also show why local approximation in optimization is stronger than ordinary regression accuracy: what matters is not just that the surrogate fits sampled values, but that it does so uniformly over the region and with the right derivative accuracy.

This viewpoint is particularly helpful for understanding learned surrogates. In a replacement strategy, the method acts directly on surrogate derivatives or on steps computed from the surrogate, so errors in gradients or curvature information can immediately affect search directions, quasi-Newton updates, and globalization logic. In that case, fully linear or fully quadratic behavior is the natural benchmark. In an assistance strategy, by contrast, the surrogate is filtered through an acceptance test based on the true objective, so the algorithm can benefit from a useful but imperfect model without requiring the same degree of uniform reliability at every iteration. This is one of the main conceptual reasons why assistance can be substantially more robust than replacement.

\subsection{Limitations of Classical Polynomial Models}

Classical polynomial models remain the theoretical backbone of model-based DFO, but they face several structural limitations in the settings that motivate learned surrogates.

First, their parameter count grows rapidly with dimension. A full quadratic model in $n$ variables has
\[
\frac{(n+1)(n+2)}{2}
\]
coefficients, so building and updating such a model becomes increasingly expensive as $n$ grows. This is not only a storage issue; it also affects the number of sample points required for stable interpolation or regression and the cost of the associated linear algebra.

Second, polynomial models are highly sensitive to sampling geometry. Their quality depends on the sample set being sufficiently well poised. When the available points are clustered, nearly collinear, or concentrated on a low-dimensional subset of the trust region, the resulting interpolation or regression system can become ill-conditioned. In practical algorithms, this situation occurs naturally because data are generated by the optimization path rather than by a carefully designed experimental design. For example, points produced by backtracking line search often lie approximately on a single ray, which yields poor coverage of orthogonal directions.

Third, the data available in modern black-box optimization are often noisy, heterogeneous, and reused across iterations. In such cases, a rigid polynomial basis may fail to exploit repeated structure effectively. Neural networks, in contrast, can absorb oversampled or irregular datasets more flexibly, and once trained they provide smooth, cheap-to-evaluate gradients. However, this flexibility should not be confused with automatic algorithmic improvement. Giovannelli et al.\ show that even when neural surrogates produce competitive value and gradient approximations, directly replacing the finite-difference gradient in a sensitive quasi-Newton framework does not reliably improve optimization performance. This suggests that the main difficulty is not merely approximation in isolation, but the interaction between approximation error, data geometry, and the update mechanism of the base solver.

These observations motivate the perspective adopted in the rest of the paper. We do not ask only whether a neural network can approximate a local objective well. Instead, we ask a more optimization-relevant question: \emph{when does a learned local model provide the type of information that the surrounding algorithm can safely exploit?} The answer depends jointly on local coverage, training objectives, smoothness of the surrogate, and crucially, whether the surrogate is used to replace a core step or merely to assist the search.

\section{Neural Networks for Local Approximation}\label{sec:nn}

\subsection{Motivation and Advantages}\label{sec:nn_motivation}
Two mechanisms explain why NN surrogates can improve local approximation in optimization. In terms of the central claim of Section~\ref{sec:intro}, these mechanisms explain why NNs \emph{can} help; whether they \emph{do} help is decided by the embedding role and the region radius studied in Sections~7--9.

\paragraph{Smoothing and reduced oscillations.}
Finite-difference gradients exhibit a fundamental bias--variance trade-off. Under smoothness, forward-difference truncation error is $O(h)$; under noise, variance scales as $O(h^{-2})$. A basic model clarifies this. Suppose we observe $\tilde f(x)=f(x)+\xi(x)$ where $\mathbb{E}[\xi]=0$ and $\mathrm{Var}(\xi)=\sigma^2$. The forward-difference estimator obeys (approximately)
\begin{equation}\label{eq:fd_noise_var}
\mathrm{Var}\big([\tilde g_h(x)]_i\big)
=\mathrm{Var}\left(\frac{\xi(x+h e_i)-\xi(x)}{h}\right)
\approx \frac{2\sigma^2}{h^2}.
\end{equation}
Thus reducing truncation error by using smaller $h$ increases noise variance dramatically. Learned surrogates can smooth because they fit a global (or at least region-wide) function class to many data points, implicitly averaging noise and producing a more stable gradient field.

\paragraph{Trend prediction and directional guidance.}
Beyond smoothing, the surrogate can infer a coherent descent direction (trend) from scattered samples. This is especially plausible when the training objective includes derivative targets: Sobolev training matches both values and gradients, improving the surrogate's local directional fidelity \cite{Czarnecki2017Sobolev,Taminiau2025Arxiv2502}. In Taminiau et al.\ the surrogate is designed to be continuously differentiable and is used via gradient descent with an Armijo-like condition, directly exploiting this directional information \cite{Taminiau2025Arxiv2502}.

\paragraph{Why these mechanisms are not automatic.}
Giovannelli et al.\ emphasize a key caveat: high-quality approximation (especially in function values or gradients) might not improve optimization if the surrogate is inserted into a sensitive part of the algorithmic loop. In particular, they report that using surrogate gradients (including NNs) inside a practical FD-BFGS step within FLE (the full-low evaluation framework of \cite{BerahasSohabVicente2023}) does not reliably improve performance across CUTEst problems \cite{GouldOrbanToint2015,Giovannelli23T027}. This motivates careful embedding and safeguards.

\subsection{Approximation Properties and Smoothness Conditions}
In optimization, we usually need surrogate derivatives. This leads to explicit smoothness constraints.

\begin{assumption}[Differentiable surrogate family]\label{ass:diff_surrogate}
The surrogate $m_\theta:\R^n\to\R$ is continuously differentiable for all $\theta$ produced by training.
\end{assumption}

Taminiau et al.\ exclude ReLU specifically because it does not yield a differentiable model in their setting and focus on smooth activations\cite{Taminiau2025Arxiv2502}. Giovannelli et al.\ compare ReLU, ELU, SiLU, Sigmoid, and Tanh, and identify SiLU as the best-performing activation in their approximation benchmarks \cite{Giovannelli23T027}.

\paragraph{First-order vs second-order approximation.}
Giovannelli et al.\ show that interpolation/regression models often provide better accuracy for second-order approximations, while NNs can be competitive for zero- and first-order approximations at high training cost \cite{Giovannelli23T027}. This is important because trust-region convergence theory relies on fully linear/quadratic properties (Section~\ref{sec:fullylinear}). If an algorithm requires accurate Hessians (fully quadratic), NN surrogates may require additional structure (e.g.\ special architectures, regularization targeting second derivatives, or explicit curvature information in training); otherwise, they may be better suited as first-order tools within safeguarded frameworks.

\subsection{Comparison with Polynomial Models}
Giovannelli et al.\ interpret activation functions as nonlinear feature maps and propose activation-enriched polynomial bases. Two representative constructions (written here in a compact ``feature list'' form) replace quadratic cross terms with activated nonlinear features \cite{Giovannelli23T027}:
\begin{align}
\tilde{\phi}(x) &= \Big\{1,\ x_1,\dots,x_n,\ x_1^2/2,\ s(x_1x_2),\dots,s(x_{n-1}x_n),\ x_n^2/2\Big\}, \label{eq:phi_tilde}\\
\hat{\phi}(x) &= \Big\{1,\ x_1,\dots,x_n,\ s(x_1),\dots,s(x_n),\ x_1^2/2,\dots,x_n^2/2\Big\}, \label{eq:phi_hat}
\end{align}
where $s(\cdot)$ is an activation function. Their findings suggest that such enrichment can ``waive the necessity'' of including cross terms in some cases, yielding fewer parameters while maintaining approximation quality \cite{Giovannelli23T027}. From a DFO viewpoint, this is a valuable insight: ML-inspired feature design can create models that sit between classical polynomials and full NNs, potentially balancing sample efficiency and stability.

Taminiau et al.\ also compare NNs with RBF surrogates trained by the same Sobolev objective (gradient term for RBF reduces to least squares); they set $\lambda=0$ for RBF and solve by a direct least-squares solver, choosing the minimum-norm solution when non-unique \cite{Taminiau2025Arxiv2502}. This highlights that ``ML helps'' need not mean ``NNs are best'': the key is the training objective and the embedding into a safeguarded loop.

\section{Training Data Selection for Local Models}\label{sec:data}

\subsection{Sampling Strategies in the Trust Region}
The surrogate model is only as good as its data. In model-based DFO, a common paradigm is to sample inside $\Ball{x_k}{\Delta_k}$ in a way that ensures geometry and stability \cite{Conn2009DFO,ConnScheinbergVicente2006TR}. Giovannelli et al.\ sample training and testing datasets uniformly in a ball $\Ball{x_0}{1}$ around a CUTEst-provided initial point $x_0$ (for approximation evaluation) \cite{Giovannelli23T027}. They also normalize data by shifting by $-x_0$ and scaling by $\Delta=\max_i\norm{x_i-x_0}$ so the data lie in $\Ball{0}{1}$ \cite{Giovannelli23T027}.

Taminiau et al.\ derive data from the base method itself: each outer iteration adds a function-value point to $F$ and adds a finite-difference gradient target to $G$ at the current iterate \cite{Taminiau2025Arxiv2502}. This is a typical ``DFO-as-data-generator'' loop: the algorithm's exploration naturally populates the region of interest, and the learned surrogate is updated online.

\paragraph{Practical sampling variants (recommended).}
In addition to uniform sampling, we recommend considering:
(i) \emph{diversity-aware sampling} (maximize minimum distances) to improve coverage, and 
(ii) \emph{poisedness-improving repairs} for polynomial/RBF surrogates \cite{Conn2009DFO}. 
When using NNs, diversity-aware sampling can reduce redundancy and improve generalization within a trust region.

\subsection{Geometry and Poisedness of Sample Sets}
For polynomial interpolation, sample geometry is formalized by poisedness \cite{Conn2009DFO}. We use a concise version.

\begin{definition}[$\Lambda$-poisedness (informal)]\label{def:poised}
Let $\mathcal{P}$ be a polynomial model space and let $Y=\{y_0,\dots,y_p\}\subset\Ball{x}{\Delta}$. The set $Y$ is \emph{$\Lambda$-poised} in $\Ball{x}{\Delta}$ if the interpolation system is well-conditioned and the corresponding Lagrange polynomials are uniformly bounded by $\Lambda$ on the ball \cite{Conn2009DFO}.
\end{definition}

Poisedness supports fully linear/quadratic guarantees for polynomial models. NNs lack an analogous linear-system criterion, but geometry is still decisive: if samples lie near a low-dimensional set (e.g.\ mostly along a line-search ray), then the learner's coverage in orthogonal directions is poor, often limiting its local generalization radius.

Giovannelli et al.\ explicitly discuss this issue for their surrogate-enhanced FD-BFGS step: line-search points in Armijo backtracking are aligned, which is problematic for interpolation/regression models. They therefore add only the first line-search point to the dataset for interpolation/regression surrogates, while for NN surrogates they add all such points \cite{Giovannelli23T027}. This motivates explicit redundancy diagnostics.

\subsection{Exploration--Accuracy Trade-offs}\label{sec:tradeoffs}
Trust regions frame the key trade-off: small $\Delta$ improves local approximation; large $\Delta$ supports exploration but may exceed the surrogate's reliable range.

To operationalize this trade-off, we recommend coupling the trust-region radius (or stepsize scale in FD methods) with surrogate diagnostics. In Section~\ref{sec:generalization} we define an effective generalization radius; in Section~\ref{sec:experiments} we estimate it empirically through region-size sweeps.

\paragraph{Redundancy diagnostics, and what became of them.}
An earlier version of this work proposed two cheap geometric diagnostics for redundancy in
the training data, an overlap ratio between consecutive training sets and a step-spacing
statistic along the backtracking ray, and built the framework of Section~\ref{sec:aras}
around them. Section~\ref{sec:exp_diag} tests both against $3202$ outer iterations and finds
that neither predicts whether a surrogate proposal will be accepted, as does a
singular-value diagnostic for directional coverage. We therefore gate on measured model
quality instead, through the statistics of \eqref{eq:aras_gate}. The definitions and the
full refutation are given in the supplementary document, since a reader building such a
method would otherwise reach for the same quantities; what the main text needs from this
subsection is only that the redundancy question was asked, answered negatively, and
replaced.

\section{Offline and Online Learning Strategies}\label{sec:learning}

\subsection{Offline Pre-trained Models}
Offline pre-training is attractive when the optimization problems arise from a task family with shared structure (e.g.\ repeated design optimization, parametric PDEs). In such settings, one can learn transferable representations and reuse them across runs. Neither primary source assumes such a distributional setting: Giovannelli et al.\ train surrogates per test problem on locally sampled datasets \cite{Giovannelli23T027}, and Taminiau et al.\ train and update surrogates online within a single run \cite{Taminiau2025Arxiv2502}.

Therefore, within the scope of our synthesis, offline models should be treated as:
(i) initializations (warm starts), or
(ii) priors/regularizers that bias the online fit toward smoothness.

A caveat is that offline generalization across objectives is fundamentally different from local interpolation on a fixed objective, and it should be tested explicitly (Section~\ref{sec:noise}); a systematic cross-task study is beyond the scope of this paper.

\subsection{Online Adaptive Training}
Online learning updates the surrogate as new points are evaluated.

\paragraph{Giovannelli et al.\ (approximation pipeline).}
For approximation studies, Giovannelli et al.\ train NN surrogates by minimizing a mean-squared error over a training dataset $D$ sampled in a ball and evaluate on a separate testing dataset sampled in the same ball. They train five times per problem to compare activations and report that SiLU achieves the best performance across their benchmark suite \cite{Giovannelli23T027}. They use a feedforward network with 2 hidden layers and $4n$ neurons per layer, train it with Adam for 300 epochs, and use ReduceLROnPlateau (factor 0.8, patience 15) based on the \emph{testing} empirical risk \cite{Giovannelli23T027}. These details matter because the surrogate's approximation quality is tied to training protocol as much as to architecture.

\paragraph{Giovannelli et al.\ (online within FLE-S).}
In their surrogate-enhanced FD-BFGS step (FLE-S), they train an NN surrogate on the accumulated dataset $D_k$ once $\abs{D_k}$ exceeds a threshold, and they use a very small number of epochs to limit overhead: 5 epochs at the first surrogate call, then 1 epoch in subsequent calls \cite{Giovannelli23T027}. They choose learning rates by grid search and set $\zeta$ (the threshold multiplier) differently for polynomial/RBF vs NN surrogates \cite{Giovannelli23T027}. These choices embody a core online-learning principle: amortize training, limit per-iteration cost, and accept that the surrogate may be imperfect.

\paragraph{Taminiau et al.\ (online, warm-start, capped).}
Taminiau et al.\ train a shallow NN surrogate with one hidden layer of width $5n$ and smooth activations, solve the training problem by L-BFGS, and warm-start subsequent trainings from previous parameters \cite{Taminiau2025Arxiv2502}. To limit cost and prevent ill-conditioning, they cap dataset sizes: $N\le 10(n+1)$ and $M\le 10$, removing the oldest points first if thresholds are exceeded \cite{Taminiau2025Arxiv2502}. These design choices are tightly coupled to their safeguarded-assistance embedding.

\subsection{Hybrid Learning Strategies}
A hybrid strategy can combine offline priors with online guarantees:
(i) pretrain a feature extractor on a task family (if available),
(ii) normalize and fine-tune online in local regions,
(iii) accept surrogate steps only when the true objective decreases sufficiently (assistance with safeguards).
This approach is conceptually aligned with the framework of \cite{Taminiau2025Arxiv2502}.

\section{Model Architecture and Complexity}\label{sec:arch}

\subsection{Neural Network Structures for Local Modeling}
We consider feedforward NNs. A one-hidden-layer (shallow) model is
\begin{equation}\label{eq:shallowNN}
m_\theta(x)=W_2\,\phi(W_1 x+b_1)+b_2,
\end{equation}
with $W_1\in\R^{q\times n}$, $b_1\in\R^q$, $W_2\in\R^{1\times q}$, $b_2\in\R$, and elementwise activation $\phi$. Taminiau et al.\ use $q=5n$ \cite{Taminiau2025Arxiv2502}. Giovannelli et al.\ use a 2-hidden-layer network with width $4n$ per layer in their approximation experiments \cite{Giovannelli23T027}.

\subsection{Model Size and Approximation Capacity}
Model size controls expressivity and training cost. The shallow model \eqref{eq:shallowNN} has
\begin{equation}\label{eq:param_shallow}
n_w = qn + q + q + 1 = q(n+2)+1
\end{equation}
parameters. With $q=5n$, $n_w\approx 5n^2+10n+1$. A two-hidden-layer network with widths $q_1=q_2=4n$ has
\begin{equation}\label{eq:param_deep}
n_w = (q_1 n + q_1) + (q_2 q_1 + q_2) + (q_2 + 1) \approx 20n^2+12n+1.
\end{equation}
Thus a deeper/wider architecture has substantially higher parameter count, potentially improving approximation at the cost of larger datasets and higher training time.

\paragraph{Complexity relative to oracle cost.}
In DFO, the relevant question is not training cost in isolation but training cost \emph{relative to the cost of a function evaluation}. If evaluations are extremely expensive (e.g.\ PDE solves), expensive training may still be worthwhile. In many numerical benchmarks, however, evaluation costs are moderate and training overhead can dominate. Giovannelli et al.\ emphasize that NN surrogates may be competitive in function-evaluation counts but at high training costs \cite{Giovannelli23T027}; conversely, Taminiau et al.\ explicitly restrict model size and dataset size to keep training overhead bounded \cite{Taminiau2025Arxiv2502}. This motivates reporting both evaluation counts and wall-clock time (Section~\ref{sec:experiments}).

\subsection{Regularization and Stability}
Regularization stabilizes training in low-data regimes:
\begin{equation}\label{eq:weight_decay}
\lambda\norm{\theta}_2^2.
\end{equation}
Giovannelli et al.\ also normalize data via shifting/scaling into $\Ball{0}{1}$, which improves numerical conditioning and can mitigate training pathologies (e.g.\ vanishing gradients) \cite{Giovannelli23T027}. Taminiau et al.\ use weight decay $\lambda=10^{-4}$ and employ Sobolev training, which introduces a gradient-mismatch penalty and can be interpreted as a curvature (diagonal Hessian) regularizer under finite differences (Section~\ref{sec:noise}) \cite{Taminiau2025Arxiv2502}.

\section{Generalization and Error Analysis}\label{sec:generalization}

\subsection{Local Generalization in Trust Regions}
This section develops the second half of our central claim: surrogate trustworthiness is a \emph{radius} question. We formalize a radius-aware notion of generalization suited to local modeling.

\begin{definition}[Local uniform error profiles]\label{def:error_profiles}
Given $x\in\R^n$ and radius $\Delta>0$, define
\begin{equation}\label{eq:err_profiles}
e_f(x,\Delta):=\sup_{u\in\Ball{x}{\Delta}}\abs{f(u)-m(u)},\qquad
e_g(x,\Delta):=\sup_{u\in\Ball{x}{\Delta}}\norm{\grad f(u)-\grad m(u)}_2.
\end{equation}
\end{definition}

\begin{definition}[Effective local generalization radius]\label{def:rgen}
Fix a floor $\epsilon_{g,0}\ge0$ and a rate $\kappa_g>0$, and set
$\bar e_g(\delta):=\epsilon_{g,0}+\kappa_g\delta$, with $\bar e_f$ defined analogously from
$\epsilon_{f,0}$ and $\kappa_f\delta^2$. The \emph{effective generalization radius} around
$x$ is
\begin{equation}\label{eq:rgen}
r_{\mathrm{gen}}(x) := \sup\bigl\{\Delta>0:\ e_f(x,\delta)\le \bar e_f(\delta)\ \text{and}\ e_g(x,\delta)\le \bar e_g(\delta)\ \text{ for every } 0<\delta\le\Delta \bigr\}.
\end{equation}
\end{definition}

The floor is not cosmetic. With $\epsilon_{g,0}=0$ the tolerance $\kappa_g\delta$ vanishes as
$\delta\to0$ while $e_g(x,\cdot)$ is nondecreasing and, for a trained network, bounded below
by its own training error, so any surrogate with nonzero error at $x$ fails the test at every
small enough $\delta$ and has $r_{\mathrm{gen}}(x)=0$ exactly. That is not an edge case here:
Section~\ref{sec:exp_region} measures the floor and finds it does not vanish. Against the
classical fully linear standard, which is $\epsilon_{g,0}=0$, the radius of every network we
fit is therefore zero, and the content of that section is how large $\epsilon_{g,0}$ must be
before the radius becomes positive and useful. Quantifying over all $\delta\le\Delta$ is what
makes the qualifying set an interval: both sides increase with $\delta$, so testing at
$\Delta$ alone would admit a surrogate that passes there and fails on a smaller ball
inside.

Defining $r_{\mathrm{gen}}$ makes explicit what is often implicit in surrogate-based optimization: a surrogate may be ``good enough'' only within a limited region. Giovannelli et al.\ report that reducing ball radii can improve approximation accuracy \cite{Giovannelli23T027}, consistent with finite $r_{\mathrm{gen}}$ that depends on curvature and sample coverage.

\subsection{Error Bounds for Neural Approximation}\label{sec:errorbounds}
The two propositions below connect training accuracy and sample coverage to uniform errors. Neither uses anything specific to neural networks: the first is Lipschitz continuity and a fill distance combined by the triangle inequality, the second the standard inexact-gradient descent condition. We offer them as an elementary analytical interpretation of what the experiments measure, not as a theory of learned models, and we do not verify that a trained network satisfies their hypotheses.

\begin{assumption}[Smoothness]\label{ass:smoothness_fg}
$f$ is continuously differentiable on $\Ball{x}{\Delta}$ and $\grad f$ is $L_f$-Lipschitz on $\Ball{x}{\Delta}$. The surrogate $m$ is continuously differentiable and $\grad m$ is $L_m$-Lipschitz on $\Ball{x}{\Delta}$.
\end{assumption}

\begin{assumption}[Sample coverage via fill distance]\label{ass:filldist}
A sample set $Y\subset\Ball{x}{\Delta}$ satisfies
\begin{equation}\label{eq:filldist}
h(Y;\Ball{x}{\Delta}) := \sup_{u\in\Ball{x}{\Delta}} \min_{y\in Y} \norm{u-y} \le \varepsilon.
\end{equation}
\end{assumption}

\begin{assumption}[Pointwise training accuracy]\label{ass:pointwise_train}
Training yields pointwise bounds on $Y$:
\begin{equation}\label{eq:pointwise_bounds}
\max_{y\in Y}\abs{f(y)-m(y)}\le \epsilon_f,\qquad
\max_{y\in Y}\norm{\grad f(y)-\grad m(y)}\le \epsilon_g.
\end{equation}
\end{assumption}

\begin{remark}[On the strength of the assumptions]\label{rem:assumptions}
Assumptions~\ref{ass:smoothness_fg}--\ref{ass:pointwise_train} are strong. Lipschitz constants of $\grad m$ for a trained network are rarely available in closed form, pointwise training errors are only observed on the sample set, and certifying a fill distance \eqref{eq:filldist} is expensive beyond a few dimensions. We therefore use the following propositions as \emph{analytical tools} rather than verifiable certificates: they identify which quantities (coverage $\varepsilon$, training errors $\epsilon_f,\epsilon_g$, smoothness constants $L_f,L_m$) control uniform model quality, they explain the empirical radius effect observed in Section~\ref{sec:experiments}, and they motivate the diagnostics of Section~\ref{sec:tradeoffs}. None of the algorithms discussed in this paper require these assumptions in order to run.
\end{remark}

\begin{proposition}[Uniform gradient error bound]\label{prop:uniform_grad}
Under Assumptions~\ref{ass:smoothness_fg}--\ref{ass:pointwise_train},
\begin{equation}\label{eq:uniform_grad_bound}
e_g(x,\Delta)\le \epsilon_g + (L_f+L_m)\varepsilon.
\end{equation}
\end{proposition}
\begin{proof}
Fix $u\in\Ball{x}{\Delta}$ and choose $y\in Y$ with $\norm{u-y}\le \varepsilon$. Then
\[
\norm{\grad f(u)-\grad m(u)}
\le \norm{\grad f(u)-\grad f(y)}+\norm{\grad f(y)-\grad m(y)}+\norm{\grad m(y)-\grad m(u)}.
\]
Apply Lipschitz bounds and Assumption~\ref{ass:pointwise_train}.
\end{proof}

\begin{proposition}[Descent direction under inexact gradients]\label{prop:descent_inexact}
Let $g_m(x)=\grad m(x)$. If $\norm{g_m(x)-\grad f(x)}\le \epsilon_g$, then
\begin{equation}\label{eq:descent_cond}
\ip{\grad f(x)}{-g_m(x)} \le -\norm{\grad f(x)}_2\big(\norm{\grad f(x)}_2-\epsilon_g\big).
\end{equation}
In particular, if $\norm{\grad f(x)}_2>\epsilon_g$, then $-g_m(x)$ is a strict descent direction for $f$ at $x$.
\end{proposition}
\begin{proof}
Use $\ip{\grad f}{-g_m}=-\norm{\grad f}^2-\ip{\grad f}{g_m-\grad f}$ and Cauchy--Schwarz.
\end{proof}

The $\epsilon_g$ of Proposition~\ref{prop:descent_inexact} is a bound on the gradient error
at the single point $x$, whereas the $\epsilon_g$ of
Assumption~\ref{ass:pointwise_train} bounds it on the training set. Any valid upper bound
may be substituted, and the one we can measure is the uniform error $e_g(x,\Delta)$ of
Definition~\ref{def:error_profiles}, which dominates the pointwise error on the ball; the
substitution is therefore conservative, and it is how Section~\ref{sec:exp_region} reads the
proposition against data.

\begin{remark}[Implication for surrogate usefulness]
Proposition~\ref{prop:descent_inexact} implies a qualitative threshold: if the surrogate gradient error is on the same order as the true gradient norm, then surrogate-based steps may fail to decrease $f$ or may yield erratic behavior. This explains why assistance strategies that test true decrease can be robust even when the surrogate is imperfect.
\end{remark}

\subsection{Robustness under Noise}\label{sec:noise}
Taminiau et al.\ provide a key structural insight: Sobolev learning with finite-difference gradients induces a curvature penalty.

\begin{proposition}[Sobolev learning as curvature regularization]\label{prop:sobolev_curvature}
(Adapted from Proposition~3.1 in \cite{Taminiau2025Arxiv2502}.)
Let $m_\theta:\R^n\to\R$ be twice continuously differentiable. Fix $z\in\R^n$ and $h>0$ such that
\begin{equation}\label{eq:interp_stencil}
m_\theta(z)=f(z),\qquad m_\theta(z+h e_\ell)=f(z+h e_\ell),\ \ell=1,\dots,n.
\end{equation}
Define $g_h(z)$ by forward differences \eqref{eq:fd_forward}. Then there exist $\zeta_1,\dots,\zeta_n\in[0,1]$ such that
\begin{equation}\label{eq:curv_pen}
\norm{g_h(z)-\grad m_\theta(z)}_2^2=\frac{h^2}{4}\sum_{\ell=1}^n\left([\hess m_\theta(z+h\zeta_\ell e_\ell)]_{\ell\ell}\right)^2.
\end{equation}
\end{proposition}

Proposition~\ref{prop:sobolev_curvature} suggests why the gradient-matching term helps: it penalises diagonal curvature along coordinate directions when the model interpolates the finite-difference stencil. The hypothesis is worth naming, because our training does not satisfy it. Equation~\eqref{eq:sobolev_loss} is a regression with weight decay, not an interpolation, so \eqref{eq:interp_stencil} holds only approximately and the identity becomes an approximate one whose error we do not control. The proposition is therefore an idealised mechanism consistent with the acceptance-rate measurement of Section~\ref{sec:exp_sobolev}, not a causal account of it.

\paragraph{On noise in gradient targets.}
Gradient targets $g(z)$ are computed from noisy function evaluations and inherit noise that is amplified by $h^{-1}$. Weight decay, dataset capping, and warm-start training (all used in \cite{Taminiau2025Arxiv2502}) mitigate overfitting to noisy gradients, but the trade-offs are problem dependent. This motivates ablations that vary (i) the number of gradient points $M$, (ii) the finite-difference stepsize policy, and (iii) whether the value term is included in the loss.

\paragraph{Cross-task generalization.}
In some applications, one cares about generalization across objective functions (tasks) under similar sampling geometry. This notion is distinct from local generalization on a single objective. We recommend testing cross-task generalization empirically by sampling multiple functions (e.g.\ from CUTEst categories) using the same sampling design and comparing achievable $e_f,e_g$ profiles; such a study is left as future work.

\section{Algorithm Design}\label{sec:algdesign}

\subsection{Learning-Augmented Trust-Region Framework}
This section develops the first half of our central claim, the role of the embedding. We
adopt a role-based taxonomy that reconciles the two primary sources. Let a base DFO method produce iterates and data; a surrogate model is trained on that data; the surrogate is used either to replace a core ingredient or to assist with safeguarded proposals.

\subsection{Neural Model Construction and Update}
We recount the training objectives used in the primary sources.

\paragraph{Value-only training (Giovannelli et al.).}
Given a dataset $D=\{(x_i,f(x_i))\}_{i=0}^{N}$, Giovannelli et al.\ train $f_{\mathrm{NN}}(\cdot;w)$ by minimizing the empirical risk
\begin{equation}\label{eq:old_mse}
\min_{w\in\R^{n_w}} L(w;D):=\sum_{i=0}^{N}\big(f(x_i)-f_{\mathrm{NN}}(x_i;w)\big)^2,
\end{equation}
and assess generalization via a testing dataset sampled from the same ball \cite{Giovannelli23T027}. They normalize data by shifting/scaling:
\[
x_i \mapsto \frac{x_i-x_0}{\Delta},\quad \Delta=\max_{1\le i\le N}\norm{x_i-x_0},
\]
so the transformed dataset lies in $\Ball{0}{1}$ \cite{Giovannelli23T027}.

\paragraph{Sobolev loss (Taminiau et al.).}
Taminiau et al.\ define $F=\{(y_i,f(y_i))\}_{i=1}^{N}$ and $G=\{(z_j,g(z_j))\}_{j=1}^{M}$, with $g(z_j)\approx \grad f(z_j)$ from finite differences, and train a differentiable surrogate $m_\theta$ by
\begin{equation}\label{eq:sobolev_loss}
\min_{\theta}\ L_{F,G}(\theta):=
\frac{1}{N}\sum_{i=1}^{N}\big(m_\theta(y_i)-f(y_i)\big)^2
+\frac{1}{M}\sum_{j=1}^{M}\norm{\grad m_\theta(z_j)-g(z_j)}_2^2
+\lambda\norm{\theta}_2^2.
\end{equation}
They set $\lambda=10^{-4}$ and solve \eqref{eq:sobolev_loss} with low-memory BFGS (L-BFGS), stopping at iteration $K$ if
\begin{equation}\label{eq:lbfgs_stop}
\norm{\grad L_{F,G}(\theta_K)}_2 \le 10^{-6}\max\{1,\norm{\grad L_{F,G}(\theta_0)}_2\},
\end{equation}
or if a budget of 1000 iterations is reached \cite{Taminiau2025Arxiv2502}. They initialize the first surrogate with He initialization (SoftPlus/SiLU) or Glorot initialization (sigmoid) and warm-start subsequent surrogates from the last trained model \cite{Taminiau2025Arxiv2502}. They cap dataset sizes $N\le 10(n+1)$ and $M\le 10$, removing the oldest points first \cite{Taminiau2025Arxiv2502}.

\paragraph{Value-removal variants and identifiability.}
If the value term is removed and one trains only on gradients, the learned model is identifiable only up to an additive constant. This can be acceptable if only directions are used; however, any acceptance rule based on predicted decreases requires fixing the offset (e.g.\ via anchoring $m_\theta(x_{\mathrm{ref}})=f(x_{\mathrm{ref}})$). This is relevant in workflows where one emphasizes gradient accuracy over value accuracy (e.g.\ some shape-optimization pipelines).

\subsection{Step Computation and Acceptance Criteria}\label{sec:stepcomp}
To keep the main text light, we present pseudocode only for the safeguarded assistance procedure (Algorithm~\ref{alg:surrogate}), which is the embedding at the heart of our central claim. The base method, due to \cite{Grapiglia2024DFQR} and simplified by Taminiau et al., its surrogate-augmented driver, and the FLE-S replacement step of Giovannelli et al.\ are summarized in words below and reproduced in full in the supplementary document \texttt{supplement-pseudocode.pdf}, which accompanies the \texttt{code/} directory.

\subsubsection*{Assistance: safeguarded surrogate procedure}
\begin{algorithm}[t]
\caption{Safeguarded surrogate steps (\textbf{assist}), adapted from Algorithm~1 in \cite{Taminiau2025Arxiv2502}}
\label{alg:surrogate}
\begin{algorithmic}[1]
\Require $v\in\R^n$; oracle access to $f$; datasets $F=\{(y_i,f(y_i))\}_{i=1}^N$ and $G=\{(z_j,g(z_j))\}_{j=1}^M$ with $(v,f(v))\in F$; constants $\sigma,\rho,\lambda,\gamma,\epsilon>0$
\Ensure $v^+$ (best point found), $t^+$ (number of successful surrogate steps), and $F^+$ (newly evaluated points)
\State Train a $C^1$ surrogate $m_\theta$ by approximately solving \eqref{eq:sobolev_loss}
\State $v_0\gets v$, $L_0\gets \sigma$, $F^+\gets \emptyset$, $t\gets 0$
\While{true}
    \State Find smallest integer $\ell_t\ge 0$ such that with
    $
    \hat v_t = v_t - \frac{1}{2^{\ell_t}L_t}\,\grad m_\theta(v_t),
    $
    the surrogate decrease condition holds:
    \[
        m_\theta(v_t)-m_\theta(\hat v_t)\ \ge\ \frac{\rho}{2^{\ell_t}L_t}\,\norm{\grad m_\theta(v_t)}_2^2
    \]
    \State Evaluate $f(\hat v_t)$ and set $F^+\gets F^+\cup\{(\hat v_t,f(\hat v_t))\}$
    \If{$f(v_t)-f(\hat v_t)\ \ge\ \frac{1}{\gamma\sigma}\epsilon^2$}
        \State Accept: $v_{t+1}\gets \hat v_t$, $L_{t+1}\gets 2^{\ell_t-1}L_t$, $t\gets t+1$
    \Else
        \State $t^+\gets t$, $v^+\gets v_t$ and \textbf{break}
    \EndIf
\EndWhile
\State \Return $(v^+,t^+,F^+)$
\end{algorithmic}
\end{algorithm}

\subsubsection*{Taminiau et al.: base method and surrogate-augmented method}
Taminiau et al.\ integrate the surrogate procedure into a finite-difference Armijo-type base method (a derivative-free gradient method); the pseudocode is reproduced from Algorithms~2 and 3 of \cite{Taminiau2025Arxiv2502}, with minimal notational changes, in the supplementary document \texttt{supplement-pseudocode.pdf}. In brief, the base method computes a forward-difference gradient whose stepsize is tied to the current scale $\sigma_k$, accepts trial points by an Armijo-type test, and doubles the scale upon rejection; the augmented variant calls the surrogate procedure (Algorithm~\ref{alg:surrogate}) after every accepted base step and continues from the best point found.

As reproduced there, these algorithms carry no termination test, and this is inherited from the source rather than an omission on our part: the analysis of \cite{Taminiau2025Arxiv2502} covers only the iterations before the hitting time $T(\epsilon)=\inf\{k:\norm{\grad f(x_k)}_2\le\epsilon\}$, and Lemma~2.2 there establishes that iteration $k$ is well defined precisely when $\norm{\grad f(x_k)}_2>\epsilon$. Near a stationary point the inner loop can fail both tests indefinitely, shrinking $h_i$ without ever returning. The pseudocode should therefore be read as describing the iterations $k<T(\epsilon)$ only; any executable version needs an external stopping rule. Our implementation (Section~\ref{sec:experiments}) caps the inner index at $i\le 40$, treats exhaustion of that cap as an approximate-stationarity or stall exit, and additionally bounds every run by the evaluation budget.

\subsubsection*{Replacement: surrogate gradient within FD-BFGS in FLE-S}
Giovannelli et al.\ propose FLE-S by modifying the FD-BFGS (Full-Eval) step: once the dataset $D_k$ exceeds a threshold, they build a surrogate and set the gradient approximation to the surrogate gradient \cite{Giovannelli23T027}. They also add a random point in $\Ball{x_k}{0.1}$ before training/building the surrogate, and they update $D_k$ differently depending on surrogate type (add-first for poly/RBF; add-all for NN). The corresponding pseudocode is reproduced in the supplementary document \texttt{supplement-pseudocode.pdf}.

\subsection{From Diagnostics to Design: the ARAS Framework}\label{sec:aras}

The analysis so far suggests a design recipe. Monitor how far the current surrogate can be
trusted, and let that decide how much of the step it is allowed to influence, rather than
fixing the surrogate's role in advance. We call this the \emph{Adaptive Radius-Aware
Surrogate} framework (ARAS).

The open question in such a design is which quantity to gate on. Our first version used the
geometric diagnostics of Section~\ref{sec:tradeoffs}, the overlap ratio $r_k$ and the
step-spacing statistic $N_{\mathrm{sp}}$. Section~\ref{sec:exp_diag} tests those against
data and finds their effect on acceptance an order of magnitude too small to gate on, so we
do not gate on them. The same
section tests a second family of candidates, computed from quantities the method already
holds, and identifies two that do predict whether the next surrogate proposal will be
accepted: the ratio $\|\grad m_k(x_k)\|/\|g_k^{\mathrm{FD}}\|$ between the surrogate and
finite-difference gradient norms, and the training residual of the surrogate on its own
data. Sorted into quartiles by the first of these, the empirical acceptance rate over the
$3196$ outer iterations at which both statistics are defined falls monotonically through
$0.484$, $0.343$, $0.106$ and $0.045$, an elevenfold drop from the lowest quartile of
$\kappa_k$ to the highest.
The framework therefore gates on
\begin{equation}\label{eq:aras_gate}
\kappa_k \;=\; \frac{\|\grad m_k(x_k)\|}{\|g_k^{\mathrm{FD}}\|},
\qquad
\varrho_k \;=\; \Bigl(\tfrac{1}{|F_k|}\textstyle\sum_{(y,f)\in F_k}(m_k(y)-f)^2\Bigr)^{1/2},
\end{equation}
a scale disagreement and a fit residual, both cheap and both validated.

The blended direction admits a guarantee, and the quantity it needs is not the surrogate's
accuracy\,--\,which Section~\ref{sec:exp_region} shows is not boundable below any useful
floor\,--\,but its scale relative to the finite-difference gradient, which is exactly the
$\kappa_k$ the gate already computes at no oracle cost.

Bounding the blending weight by
\begin{equation}\label{eq:alpha_rule}
0\ \le\ \alpha_k\ \le\ \frac{1-\theta}{1+\kappa_k},\qquad \theta\in(0,1),
\end{equation}
gives three things, proved in the supplementary document. The blended
direction satisfies $\ip{g_k}{d_k}\le-\theta\norm{g_k}^2$ and $\norm{d_k}\le2\norm{g_k}$
(Proposition). Backtracking on the true objective terminates in
$\lceil\log_2(4L\bar\beta/((1-c)\theta))\rceil_+$ halvings with
$\beta_k\ge\min\{\bar\beta,(1-c)\theta/(8L)\}$, provided the working scale satisfies
\begin{equation}\label{eq:scale_cond}
2^{i_k}\sigma_k\ \ge\ \frac{L}{(1-c)\theta},\qquad c\in(0,1),
\end{equation}
which is equivalent to the finite-difference error obeying
$e_k\le\frac{(1-c)\theta}{4}\norm{g_k}$ (Lemma). And ARAS then requires at most
$\mathcal{O}(n\epsilon^{-2})$ evaluations to reach $\norm{\grad f(x_k)}_2\le\epsilon$, the
base method's order, with a constant no worse: at $\alpha_k\equiv0$ the iteration is the base
method itself (Theorem).

Three points. Condition \eqref{eq:alpha_rule} is computable: $\kappa_k$ costs nothing beyond
the surrogate gradient already formed. It is not vacuous: over the $3202$ logged iterations
the cap has median $0.247$ at $\theta=1/4$ and permits $\alpha_k\ge0.1$ at $63\%$ of them.
And \eqref{eq:scale_cond} is about the base method rather than the network: with
$h_i=2\epsilon/(5\sqrt n\,2^i\sigma_k)$ the forward-difference error obeys
$e_k\le L\epsilon/(5\cdot2^{i_k}\sigma_k)$ while the inner loop exits only with
$\norm{g_k}\ge4\epsilon/5$, so $e_k/\norm{g_k}\le L/(4\cdot2^{i_k}\sigma_k)$; the condition
asks the working scale to have grown past a fixed multiple of $L$, which is what the
$\sigma_k$ update does. What the theorem does not give is an \emph{improvement} in rate.

The gate is calibrated rather than guessed, and the answer is not the expected one: on a
cheap deterministic oracle the calibrated gate never fires. A rejected proposal costs one
evaluation, an accepted one saves the $n+1$ of a finite-difference gradient, so gating pays
only where acceptance falls below $1/(n+1)$, and on this benchmark it does not
(Section~\ref{sec:exp_diag}).

\paragraph{Running the blended step.} A theorem about an algorithm nobody has run is worth
little, so we implemented the framework (pseudocode in the supplementary document) with $\theta=1/4$, $c=1/2$ and
$\alpha_{\max}=1/2$, taking the blended step at every outer iteration and following it with
the same safeguarded loop as elsewhere. Over the $117$ instances it solves $74$ at
$\tau=10^{-5}$ against $99$ for plain safeguarded assistance and $68$ for the base method,
with $464$ of $1339$ blended steps accepted at a mean $\alpha_k$ of $0.199$. The blended
step is convergent, as the Theorem says, and it is not worth its price: each
attempt spends backtracking evaluations on the true objective, and the two thirds that fail
buy nothing the surrogate loop would not have found more cheaply.

The gate calibration says the same from the other direction, and the two together settle
what ARAS should be: the instance surviving both tests has $\alpha_k\equiv0$ and is exactly
the safeguarded assistance benchmarked throughout. We keep the theorem because it certifies
that opening the gate cannot destroy convergence, and report the experiment because it says
when to leave it shut.

\begin{remark}[What is and is not covered]\label{rem:aras_safe}
The Theorem covers the blended step and the outer iteration; two parts of the
framework it does not cover should be named. The radius test
$r_{\mathrm{gen}}(x_k)\ge\eta_{\mathrm{rad}}\Delta_k$ uses a computable proxy for a quantity
Definition~\ref{def:rgen} does not make computable, so it can only make the gate more
conservative, never less; the argument above holds whether or not it fires because
\eqref{eq:alpha_rule} is enforced in either branch. And the inner surrogate loop of
Algorithm~\ref{alg:surrogate} that follows an accepted blended step contributes the finite
termination result and the gain factor of Theorem~\ref{thm:eta},
neither of which the blended step disturbs, since each of its own steps is separately
validated on the true objective.
\end{remark}

\subsection{Convergence Analysis}
\paragraph{Assistance paradigm: finite termination and surrogate gain.}
Taminiau et al.\ prove that the surrogate procedure terminates in finite time under a lower-bounded objective assumption, and they derive evaluation-complexity improvements quantified by a surrogate-gain factor \cite{Taminiau2025Arxiv2502}. We restate the essential components.

\begin{assumption}[Lower bounded objective]\label{ass:lower_bounded}
$f$ is bounded below by $f_{\mathrm{low}}\in\R$.
\end{assumption}

Algorithm~\ref{alg:surrogate} then terminates after finitely many successful surrogate steps: each accepted step decreases $f$ by at least $\epsilon^2/(\gamma\sigma)$, so infinitely many acceptances would drive $f$ below $f_{\mathrm{low}}$ (Lemma~2.1 of \cite{Taminiau2025Arxiv2502}).

\begin{theorem}[Surrogate gain factor; Theorem~2.4 of \cite{Taminiau2025Arxiv2502}]\label{thm:eta}
Let $f$ be bounded below (Assumption~\ref{ass:lower_bounded}) with $L$-Lipschitz gradient, let
\[
T(\epsilon):=\inf\{k\in\mathbb{N}:\ \norm{\grad f(x_k)}_2\le\epsilon\},
\]
and let $t_k$ be the number of successful surrogate steps taken at outer iteration $k$. Define the \emph{average} number of successful surrogate steps per outer iteration over the first $T$ iterations,
\begin{equation}\label{eq:Sdef}
S(T):=\frac{1}{T}\sum_{k=0}^{T-1}t_k ,
\end{equation}
and the surrogate gain
\begin{equation}\label{eq:eta}
\eta(S)=\frac{1+\frac{S}{2(n+1)}}{1+S}.
\end{equation}
Then the number $FE(\epsilon)$ of evaluations of $f$ required by the surrogate-augmented method satisfies
\begin{equation}\label{eq:fe_bound}
FE(\epsilon)\ \le\ 4\,\eta\!\left(S(T(\epsilon))\right)(n+1)\,C_{\max}\,(f(x_0)-f_{\mathrm{low}})\,\epsilon^{-2}
\;+\;\log_2\!\left(\tfrac{\sigma_{\max}}{\sigma_0}\right)(n+1)\;+\;T(\epsilon),
\end{equation}
with $\sigma_{\max}$ and $C_{\max}$ the problem constants defined in \cite{Taminiau2025Arxiv2502}.
\end{theorem}

The statement should be read with three qualifications. First, $S(T)$ in \eqref{eq:Sdef} is an \emph{average per outer iteration}, not a total count: $\eta$ is driven down only if the surrogate succeeds repeatedly at a typical iteration, and $S(T)\ge n$ is what yields the favourable regime $\eta\le 3/(2(n+1))$. Second, the gain multiplies only the leading $\epsilon^{-2}$ term of \eqref{eq:fe_bound}; the remaining two terms are unaffected, so the improvement is in the constant of the dominant term rather than in the $\mathcal{O}(n\epsilon^{-2})$ rate itself. Since $\eta(0)=1$, the bound never degrades relative to the base method. Third, and this limits how the result may be used, $S(T(\epsilon))$ is measured on the very run the bound describes. The statement is therefore a posteriori: before running, the only value of $\eta$ one is entitled to assume is $\eta(0)=1$, at which the bound is the base method's own. It certifies that observed surrogate success translates into a smaller constant, not that success will occur.

\paragraph{Replacement paradigm: why guarantees are harder.}
In replacement methods, surrogate gradient errors propagate through quasi-Newton updates and line-search decisions: an error in $g_k$ contaminates both the direction $p_k=-H_kg_k$ and the curvature pair $(s_k,y_k)$ that updates $H_k$, so the damage accumulates. Giovannelli et al.\ report that surrogate modeling ``hardly'' improves their state-of-the-art FLE method when used to approximate gradients, even with NN surrogates \cite{Giovannelli23T027}. Without uniform region-wise error control or acceptance safeguards tied to the true objective, replacement can underperform even where pointwise approximation looks good.

\section{Numerical Experiments}\label{sec:experiments}

This section reports the numerical evidence. The experiments are small, and they hold the network class, the training pipeline and the base method fixed across variants, so a difference between two of them cannot be blamed on implementation. That is weaker than varying one mechanism at a time, and we do not claim the stronger thing: two embeddings differ in more than their nominal role, since changing it also changes the evaluation cost per iteration, the number of proposals, the fallback behaviour and the trajectory along which the surrogate is trained, and those in turn change when a run effectively stops. The deterministic runs of Sections~\ref{sec:exp_role} to \ref{sec:exp_sobolev} use a single run per instance. What follows is therefore evidence consistent with the mechanisms we have described and a measurement of how large their effects are in a controlled setting, not a causal isolation of the role variable and not a benchmark against state-of-the-art model-based DFO software. The complete implementation, the $3146$ raw result files behind every number reported
here, and the scripts that turn one into the other are available at \url{https://github.com/chengkuobian/neural-surrogates-dfo}: plain
Python with NumPy and Matplotlib only, the surrogate, its analytic parameter gradients, the
Adam trainer and all four solvers implemented from scratch, with a \texttt{README} giving
the exact commands, parameter values, caps and fallback rules, and a table mapping each
script to the table or figure it produces.

\subsection{Setup}\label{sec:exp_setup}

\paragraph{Problems.} We use thirteen problems from the Mor\'e--Garbow--Hillstrom collection
\cite{MoreGarbowHillstrom1981}: sphere, chained Rosenbrock, Powell singular, trigonometric,
Broyden tridiagonal, Broyden banded, variably dimensioned, discrete boundary value,
discrete integral equation, Brown almost-linear, penalty I, extended Freudenstein--Roth
and extended Wood. Each is taken in dimensions $n\in\{4,8,16\}$ from three starting
points, the standard one, the standard one scaled by ten, and a random perturbation of
it, giving $117$ instances. These are $117$ instances of thirteen functions, not $117$
independent problems; counts across them are not independent samples, which is why the
intervals reported below resample problems rather than instances. The budget is $100$ simplex
gradients, that is $100(n+1)$ evaluations, the convention of \cite{MoreWild2009}.

\paragraph{Methods.} Six solvers are compared.
\begin{enumerate}[leftmargin=2em]
\item \textbf{FD-base}: the derivative-free quadratic regularization method of \cite{Grapiglia2024DFQR}, in the simplified finite-difference Armijo-type form \cite{Taminiau2025Arxiv2502} use as their base method, with the inner loop capped for practicality, $\epsilon=10^{-3}$ and $\sigma_0=1$.
\item \textbf{NN-assist}: the safeguarded assistance embedding (Algorithm~\ref{alg:surrogate}, driven as in \cite{Taminiau2025Arxiv2502}) with a Sobolev-trained surrogate.
\item \textbf{NN-assist-val}: identical to NN-assist except that the gradient term is removed from the loss \eqref{eq:sobolev_loss} (value-only training).
\item \textbf{DFO-TR}: a model-based trust-region method with minimum-Frobenius-norm
quadratic interpolation on $2n+1$ points, exact trust-region subproblem solves, and a
geometry-improving step when the interpolation set is poorly poised. This is the standard
scheme behind \textsc{newuoa} and \textsc{bobyqa} \cite{Powell2006NEWUOA} and the DFO-TR
family of \cite{Conn2009DFO}; implemented alongside the rest because our compute environment has no network access;
its interpolation conditions and subproblem solutions are checked against direct
computation, and Section~\ref{sec:exp_limits} reports a comparison against released
Py-BOBYQA run on the authors' own machine.
\item \textbf{DFO-TR+NN}: the same method with the assistance loop attached. It refits the
surrogate once every $n+1$ accepted steps rather than at every step, because a
trust-region iteration costs one evaluation where a finite-difference outer iteration costs
$n+1$; this equalises training cost per evaluation so the comparison is not confounded by
refit frequency.
\item \textbf{NN-replace}: surrogate-gradient replacement with line search and finite-difference fallback. Once $|F|\ge 2(n+1)$ the FD gradient is \emph{replaced} by $\grad m_\theta(x_k)$, saving $n$ evaluations per iteration; backtracking still acts on the true objective and repeated failure triggers an FD fallback. It is not an unsafeguarded method, so the contrast with assistance is not safeguarded versus unsafeguarded: the two also differ in evaluation cost per iteration and in how often the surrogate is consulted. Keeping the skeleton fixed removes quasi-Newton curvature updating as a confounder; how a surrogate gradient would behave inside an FD-BFGS step such as FLE-S is a separate question that this experiment does not address, and that \cite{Giovannelli23T027} studies directly.
\end{enumerate}
The surrogate is a one-hidden-layer softplus network of width $6n$ with inputs and targets normalized per training, weight decay $\lambda=10^{-4}$, trained from a fresh He initialization by full-batch Adam ($900$ iterations, the setting Section~\ref{sec:exp_arch} selects); datasets are capped at $N\le 10(n+1)$ and $M\le 10$ with oldest-first removal, following \cite{Taminiau2025Arxiv2502}, as are $\sigma_0=1$ and $\lambda=10^{-4}$. We depart from it in five settings, all listed
here because two of them change how readily a surrogate step is taken: Adam for $900$
iterations rather than L-BFGS, width $6n$ rather than $5n$, $\epsilon=10^{-3}$ rather than
$10^{-5}$, $\sigma_{\min}=10^{-10}$ rather than $10^{-2}$, and, in the acceptance machinery
itself, $\rho=0.25$ rather than $10^{-4}$ and $\gamma=100$ rather than $25/2$. The last two
make a surrogate step harder to accept than in the original, so our acceptance rates are
lower than that design would give and the advantage we measure for assistance is
conservative. All of them apply identically to every surrogate variant, so none affects the
role comparison, which is a difference between variants rather than an absolute level. All solver runs are deterministic; the region-size study below uses three random seeds and reports medians.

\paragraph{Metrics.} The primary metric throughout is the number of true function
evaluations, for the reason the problem class supplies: in the simulation-based setting
these experiments are about, one evaluation is a model run and everything else is free by
comparison, so sample efficiency is what a method is judged on. We report
(i) evaluations to reach a target accuracy, (ii) the success rate across problems as data
profiles, and (iii) the surrogate acceptance rate.

Wall-clock time appears in Sections~\ref{sec:exp_time} and \ref{sec:discussion} and should
be read for what it is. Every solver here is unoptimised pure NumPy, so absolute times are
properties of this implementation rather than of the algorithms and we draw no ranking from
them. What they are used for is a question implementation speed does not affect: at what
oracle cost the evaluations a surrogate saves outweigh the time it takes to fit, and how
that threshold moves with $n$ when the parameter count grows quadratically and the training
set does not. Probe-set errors and redundancy measures are likewise reported only where they
explain a primary metric. Accuracy is measured in the standard way \cite{MoreWild2009}: with $x_N$ the best point
after $N$ evaluations, $x_0$ the start and $f_L$ the smallest value any solver in the pool
attains, a run reaches accuracy $\tau$ at the first $N$ with
$f(x_0)-f(x_N)\ge(1-\tau)\bigl(f(x_0)-f_L\bigr)$, and $N_{a,p}=+\infty$ if the budget runs out
first. The data profile $\delta_a(\beta)$ is the fraction of problems with
$N_{a,p}\le\beta(n_p+1)$, that is, solved within $\beta$ groups of $n_p+1$ evaluations
\cite{MoreWild2009}; performance profiles \cite{DolanMore2002} give the same ranking on this
test set and are omitted.

\subsection{Role of the Embedding: Replacement versus Assistance}\label{sec:exp_role}

Table~\ref{tab:role_results} reports how many of the $117$ instances each method solves at three accuracy levels, and Figure~\ref{fig:data_profiles} shows the corresponding data
profiles.

\begin{table}[t]
\centering
\caption{Extended benchmark, $117$ instances solved to accuracy $\tau$ within $100$
simplex gradients, with $f_L$ taken over all eight variants below, released Py-BOBYQA
included. The last column counts instances on which a method returns the lowest final value
of the eight. Surrogate variants use the width-$6n$, $900$-iteration network of
Section~\ref{sec:exp_arch}; the last row repeats safeguarded assistance at the width-$3n$,
$300$-iteration setting to show what undertraining costs.}
\label{tab:role_results}
\small
\begin{tabular}{l ccc c}
\toprule
Method & $\tau=10^{-1}$ & $10^{-3}$ & $10^{-5}$ & lowest $f$\\
\midrule
FD base method                    & 109 & 93 & 67 & 8\\
NN assist (Sobolev)               & 113 & 100 & 84 & 7\\
NN assist (value-only)            & 109 & 94 & 66 & 1\\
NN replace                        & 110 & 100 & 65 & 0\\
\midrule
DFO-TR (ours)                     & 113 & 102 & 88 & 33\\
DFO-TR $+$ NN assist              & 113 & 103 & 86 & 10\\
Py-BOBYQA (released)              & \textbf{117} & \textbf{116} & \textbf{103} & \textbf{55}\\
\midrule
NN assist, width $3n$, $300$ it.\ & 110 & 100 & 75 & 3\\
\bottomrule
\end{tabular}
\end{table}

\begin{figure}[t]
\centering
\includegraphics[width=0.58\linewidth]{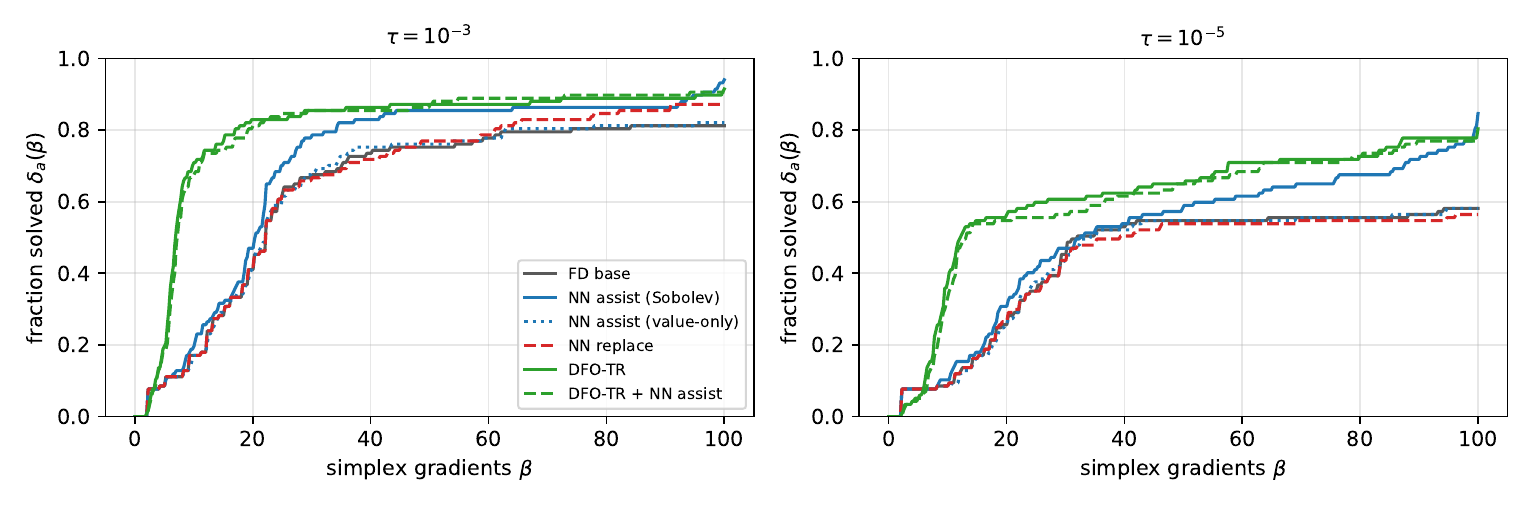}
\caption{Data profiles $\delta_a(\beta)$ over the $117$ instances, with $f_L$ taken over
all six methods, all surrogate variants at the width-$6n$, $900$-iteration setting.
Safeguarded assistance separates from its own base method as the tolerance tightens and
tracks the model-based trust-region solver, gradient replacement is the weakest
embedding, and attaching the surrogate to the trust-region solver does not help.}
\label{fig:data_profiles}
\end{figure}

Table~\ref{tab:role_results} scores every variant against released Py-BOBYQA
\cite{Powell2006NEWUOA} rather than our own interpolation code, and the choice matters. The
Mor\'e--Wild criterion asks a solver to reach $f_0-\tau(f_0-f_L)$, so a released solver
ending two orders of magnitude lower moves $f_L$ enough to reorder the methods above it:
scored among the surrogate variants alone, assistance leads our trust-region method $99$ to
$94$; adding the released solver reverses this to $84$ against $88$. We report the second,
because $f_L$ should be the best value anyone attains, and record the first because it is
the comparison a paper that never installs the released code would make.

What the surrogate does for the method it is attached to is not in doubt at either
reference. Safeguarded assistance lifts its own base method from $67$ instances to $84$ at
$\tau=10^{-5}$, from $93$ to $100$ at $10^{-3}$ and from $109$ to $113$ at $10^{-1}$: the
margin is largest exactly where the base method has stalled and least where it succeeds
anyway, which is the smoothing argument of Section~\ref{sec:nn_motivation} and the reason a
stalled finite-difference method still has a dataset worth fitting. Seventeen instances of
$117$ is the effect the rest of this paper is about.

The last row of the table is the same embedding at the width-$3n$, $300$-iteration network
an earlier version of this paper used, and it matters for how the rest should be read: that
setting reaches $75$ where the trained one reaches $84$. A third of what looked like a limit
of the mechanism was a limit of the fit. Every surrogate number reported from here on uses
the trained setting.

\paragraph{How much of that is the base method's weakness?} Almost all of it. Both
interpolation methods are ahead of assistance at every tolerance, our own by four instances
at $\tau=10^{-5}$ and the released one by nineteen. The last column sharpens it: Py-BOBYQA
returns the lowest of the eight final values on $55$ instances and our trust-region method on
$33$, against $7$ for the surrogate. The gap in final values is larger than
the gap in counts: the paired $\log_{10}$ ratio between Py-BOBYQA and the surrogate has
median $-2.321$ with interval $[-5.189,-0.168]$, so the released solver typically ends more
than two orders of magnitude lower, and it returns the lowest value of the four on $57$ of
the $117$ instances against the surrogate's $8$. A trained surrogate on a weak base method
is a large improvement on that base method and is not a substitute for a well-engineered
interpolation model.

Nor does it help one. Attaching the same trained network to our trust-region solver gives
$86$ against the $88$ it reaches alone, which is the room condition of
Section~\ref{sec:exp_noise} seen from the other side: a method that already builds a
quadratic from the same evaluations has nothing left for a surrogate to claim.

\paragraph{One problem where it does win.} The aggregate hides a case worth isolating,
because it is the only one in this set where a learned local model beats a well-engineered
interpolation model outright. On extended Wood the surrogate solves $7$ of its $9$ instances
at $\tau=10^{-5}$ against $4$ for Py-BOBYQA, $4$ for our trust-region method and $4$ for
its own base: it is the only one of the eight variants that gets past four, and it returns
the lowest final value of all eight on $3$ of the nine, once by $2.3$ orders of magnitude. Extended Wood is a coupled
quartic whose minimum sits at the end of a narrow curved valley, the regime in which a
minimum-Frobenius-norm quadratic is a poor model at any radius the method works at, and it
is exactly the regime the smoothing argument of Section~\ref{sec:nn_motivation} points to.

We resist generalising, for two reasons the same data supply. Chained Rosenbrock and
extended Freudenstein--Roth have curved valleys of the same kind and the surrogate is behind
on both. And the approximation-level statistic that ought to predict the win does not: the
ratio of the quadratic's uniform gradient error to the network's is $0.89$ at $n=16$, the
instance won by the largest margin, so by that measure the quadratic is the better local
model exactly where the surrogate wins. One problem is an existence result, not a rule. The
set of objectives on which a safeguarded neural surrogate beats released interpolation
software is not empty, and finding a diagnostic that identifies it in advance is the
question this paper leaves in the best shape for someone to answer.

\paragraph{Does the picture change with dimension?} The obvious reply is that $n\le16$ is
where interpolation is at its strongest, since a minimum-Frobenius-norm quadratic needs only
$2n+1\le33$ points, and that a learned model should come into its own where that number
grows. This deserves a direct test rather than a limitation note, so we repeated the comparison at
$n=32$, $n=64$ and $n=128$, where a quadratic needs $65$, $129$ and $257$ points, on all
thirteen problems under a budget of $50(n+1)$, and added released Py-BOBYQA
\cite{Powell2006NEWUOA} at the two largest so that the comparison does not rest on our own
interpolation code.

Adding it changed the answer, and the episode bears on how such comparisons should be read.
Our own trust-region implementation degrades with dimension: at $\tau=10^{-5}$ it leads
safeguarded assistance $88$ to $84$ at $n\le16$, leads
$11$ to $8$ at $n=32$, and at $n=64$ solves $5$ of $13$ against $7$ for assistance. Each of
these pairs is scored against the solvers available at that dimension, eight at $n\le16$,
four at $n=32$ and $n=64$ and four at $n=128$, so a pair may be compared internally but the
counts should not be read across dimensions as a single series. Its last improvement on the
median instance at $n=64$ comes at evaluation $2024$ of $3250$ against $2977$. Taken alone
this looks like the crossover the motivation for learned surrogates predicts, arriving where
the bookkeeping of a $129$-point interpolation set begins to dominate the budget.

Released Py-BOBYQA shows the reading is wrong: on the same thirteen problems at $n=64$ it
solves $12$ of $13$, far ahead of assistance at $7$ and of our implementation at $5$, and
$12$ against $9$ and $8$ at the looser $\tau=10^{-3}$, so its lead is not an artefact of
demanding high accuracy. The degradation is a property of our code, not of interpolation,
since a solver that maintains its model by rank-one updates and manages geometry carefully
loses nothing at this dimension. Our implementation is an adequate stand-in at $n\le16$,
where its median $\log_{10}$ difference from the released solver is $0.000$
(Section~\ref{sec:exp_limits}), and is not one at $n=64$. We report this rather than quietly
dropping the earlier claim, because the same trap awaits anyone who benchmarks against a
reimplementation without checking it at the dimension of interest.

At $n=128$, where a quadratic needs $257$ points and the budget allows $6450$ evaluations,
the released solver solves $10$ of $13$ at $\tau=10^{-5}$ and returns the lowest value on
$8$; assistance solves $8$ and returns the lowest on $3$; the finite-difference base method
solves $6$. Training the network properly is worth two instances here, the same embedding at
the width-$3n$ setting solving $6$ and never returning the best value, but it does not close
the gap. The pattern across the four dimensions is therefore not that the surrogate's
prospects improve as $2n+1$ grows. It is that a trained surrogate falls behind a
well-engineered interpolation method at every dimension, by nineteen instances of $117$ at
$n\le16$ and by two of thirteen at $n=128$.

The conclusion is uniform across every dimension we can test. A trained surrogate on a weak
base method is not a weak method: it lifts that base by seventeen instances of $117$, and on
extended Wood it beats released interpolation software outright. What it does not do is
match that software in aggregate, descend as far on most instances, or add anything to a
method already building a quadratic from the same evaluations.

Second, replacement is behind the base method where the tolerance separates them, $65$
against $67$ at $\tau=10^{-5}$, and it is behind at the trained setting, so the loss is not
a training artefact. It also returns the lowest of the eight final values on no instance at
all. Across the
benchmark it accepted no surrogate-gradient step that survived the true-objective test and
triggered $678$ finite-difference fallbacks. The fallback count is a direct measurement and
the natural explanation for the loss, though it is an observed correlate rather than an
independently manipulated variable.

Third, at the loose tolerance the methods are within five instances of each other.
Safeguarding makes assistance close to free: a rejected proposal costs one evaluation, so on
instances where the base method already succeeds quickly there is nothing to gain and little
to lose.

\subsection{Region Size and the Generalization Radius}\label{sec:exp_region}

To probe the radius half of the claim outside the optimization loop, we sample $N=10(n+1)$
training points uniformly in $\Ball{x_0}{\Delta}$, add forward-difference gradient targets
at $M=10$ of them, train surrogates with the Sobolev and the value-only loss, and measure
the uniform errors $e_f,e_g$ of Definition~\ref{def:error_profiles} on $300$ independent
probe points in the same ball. The sweep covers all thirteen problems at $n\in\{4,8,16\}$ and seven radii from $10^{-2}$
to $3.2$ with five seeds each, $273$ cells in all; reference gradients are central
differences at $h=10^{-6}$, accurate to about $10^{-9}$ here. Since the problems differ by
orders of magnitude in scale we report $e_g$ relative to the median $\norm{\grad f}$ over
the probe points, which makes the cells comparable, and intervals resample problems.

\begin{figure}[t]
\centering
\includegraphics[width=0.58\linewidth]{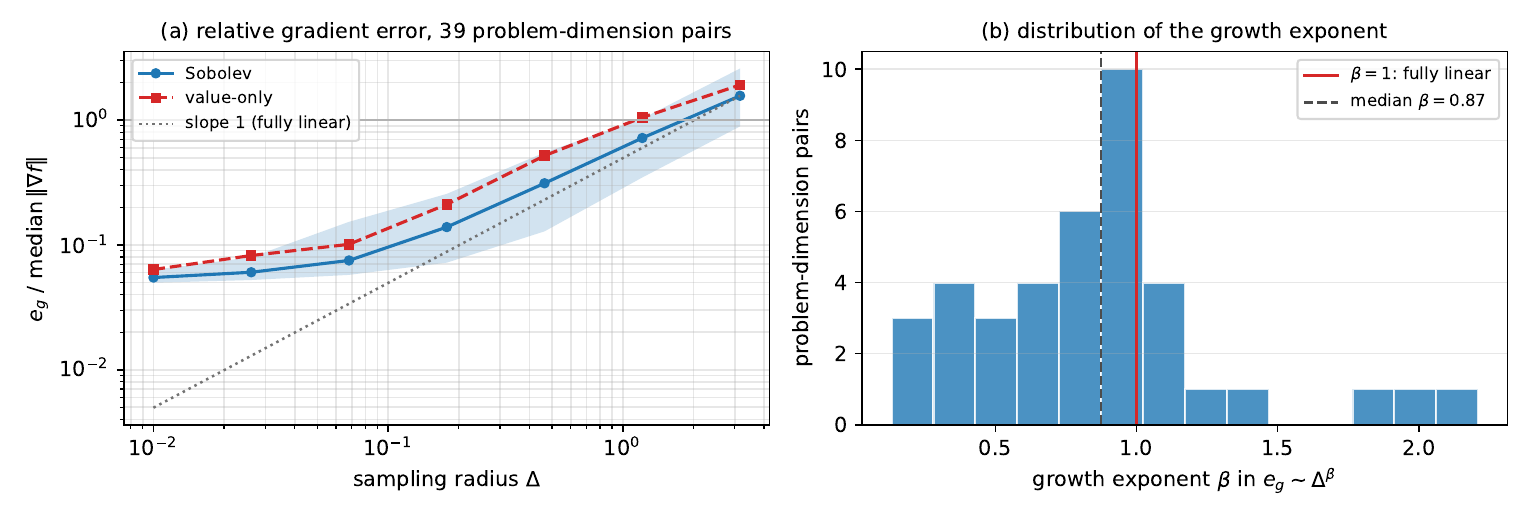}
\caption{Region-size sweep, $273$ cells. (a) Uniform gradient error relative to the median
gradient norm in the same ball, median over cells with a bootstrap band over problems; the
dotted line is the slope a fully linear model would have. (b) Distribution over the $39$
problem--dimension pairs of the fitted exponent $\beta$ in $e_g\sim\Delta^{\beta}$;
$69\%$ fall below the fully linear rate $\beta=1$, and those above it still carry a
nonzero floor at small $\Delta$.}
\label{fig:region_size}
\end{figure}

Figure~\ref{fig:region_size} shows the two effects predicted by
Section~\ref{sec:generalization}. The relative gradient error rises from $0.055$, with
interval $[0.050,0.063]$, at $\Delta=10^{-2}$ to $1.57$, with interval $[0.89,2.60]$, at
$\Delta=3.2$: at the largest radius the surrogate's gradient error exceeds the gradient it
is estimating, which by Proposition~\ref{prop:descent_inexact} is exactly where descent can no
longer be guaranteed. With a fixed sample budget the fill distance $\varepsilon$ scales with $\Delta$ and the
coverage term of Proposition~\ref{prop:uniform_grad} takes over, so the surrogate has a
finite effective generalization radius (Definition~\ref{def:rgen}).

The second effect is the more consequential, and the larger sweep sharpens it. Shrinking
the radius does not drive the error to zero. Over the smallest three radii, a factor of
seven in $\Delta$, the relative error moves only from $0.055$ to $0.075$; fitting
$e_g\sim\Delta^{\beta}$ on each problem--dimension pair gives a median exponent of
$\beta=0.87$ with interquartile range $[0.59,1.01]$, and $69\%$ of the $39$ pairs grow more
slowly than linearly. A fully linear model requires $e_g=O(\Delta)$, that is $\beta\ge1$
with the constant carrying no floor; what we observe instead is the training-error floor
$\epsilon_g$ of Proposition~\ref{prop:uniform_grad} dominating at small radii. An NN surrogate trained on an optimization-path-sized dataset is therefore \emph{not}
automatically fully linear at small radii, precisely the regime a replacement embedding
needs it to be, consistent with the caution of \cite{Giovannelli23T027}. Sobolev training reduces the relative error by a factor of $1.16$ to $1.66$ across radii,
most at moderate $\Delta$, but does not move either the small-radius floor or the
large-$\Delta$ wall.

\paragraph{Does depth move the floor?} The architecture is fixed at one hidden layer
throughout, so the floor could in principle be a property of that choice. Repeating the sweep with two hidden layers, the depth
\cite{Giovannelli23T027} use, at the same width as our single-layer network, over the same $273$ cells, moves it the wrong way: the ratio of relative gradient errors, two layers over one, is $1.069$ with
interval $[1.04,1.15]$ at $\Delta=10^{-2}$ and $1.042$ with $[1.01,1.12]$ at
$\Delta=2.6\times10^{-2}$, both excluding one, and it falls to $0.963$ with $[0.93,0.99]$
only at $\Delta\approx1.2$. Depth buys a few percent at radii where the error already exceeds the gradient it is
estimating and costs a few percent exactly where full linearity would be needed. Both
networks are trained value-only here, since the Sobolev loss's analytic parameter gradients
are derived for one hidden layer and the second-order terms a second layer introduces are a
source of silent error we chose not to risk; Section~\ref{sec:exp_sobolev} measures what
dropping the gradient term costs, so the two effects can be read together. Within that
qualification, the small-radius floor is not an artefact of shallowness.

That comparison holds the sample set fixed, so a second layer could still pay once the
surrogate steers the iteration. Run inside the assisted method on all $117$ instances it
does not: at $\tau=10^{-3}$ the depths are indistinguishable, $111$ instances against $112$,
and at $\tau=10^{-5}$, where the floor binds, the deeper network solves $97$ against $101$
and ends lower on $24$ instances. Acceptance barely moves, $0.199$ ($591/2967$) against
$0.204$ ($722/3532$), so the extra layer proposes more steps at the same success rate rather
than steps the test rejects less often. Depth is not the missing ingredient at either level,
the conclusion Section~\ref{sec:exp_arch} reaches for width and training length.

\subsection{Sobolev versus Value-Only Training}\label{sec:exp_sobolev}

The loss ablation connects the two previous experiments to the training signal. Across the
$117$ instances the surrogate acceptance rate falls from $0.703$ ($7488/10650$) with the
Sobolev loss to $0.148$ ($496/3354$) with value-only training, and the gap is not something
more training closes: both arms use the width-$6n$, $900$-iteration network. The effect on
the optimisation is as large as the effect of the embedding role. Value-only assistance
solves $66$ instances at $\tau=10^{-5}$ against $84$ for Sobolev assistance, below the base
method's $67$ rather than anywhere between the two. Because the two arms differ in one term of the loss and nothing else, the
comparison can be made instance by instance rather than through a pooled reference set:
the median of $\log_{10}$ of the ratio of final objective values, value-only over Sobolev,
is $+0.099$ with bootstrap interval $[+0.019,+0.214]$, and Sobolev is lower on $77$ of the
$117$ instances against $25$ the other way, $15$ tied ($p=2.5\times10^{-7}$, sign test).
The typical instance gains about a fifth of a decade, which is modest; the count is what
the claim rests on. This matches Proposition~\ref{prop:sobolev_curvature} and the trend-prediction discussion
of Section~\ref{sec:nn_motivation}: without gradient information the surrogate may fit the
sampled values well and still fail to produce descent directions that survive the
true-objective test. On the mildly nonlinear trigonometric and Broyden problems value-only
assistance does still help, the safeguard converting even a mediocre surrogate into a
no-worse-than-base option.

\subsection{Noise, and the Room the Base Method Leaves}\label{sec:exp_noise}

Simulation output is random, so the experiments above miss the feature that matters most in
practice. We repeat the comparison with an oracle returning $f(x)\,(1+\sigma z)$, where
$z\sim\mathcal{N}(0,1)$ is drawn afresh at every call: the relative-noise model used in the
noisy variants of the Mor\'e--Wild collection \cite{MoreWild2009}. Progress is scored on the
\emph{true} $f$ at the best point returned, so a solver gains nothing from a lucky read. The sweep covers ten problems at $n\in\{4,8\}$ and ten noise levels from $\sigma=0$ to
$\sigma=10^{-2}$, with five paired seeds per cell, $200$ cells, $1000$ paired comparisons and $2000$ solver runs in all; at a
given seed both methods meet the same noise stream, so every comparison is paired. We use $\epsilon=10^{-1}$ and a budget of $200(n+1)$ evaluations;
at the $\epsilon=10^{-3}$ of the deterministic study every solver, the base method included,
stalls immediately at any nonzero noise, which is itself part of what follows.

\begin{figure}[t]
\centering
\includegraphics[width=0.58\linewidth]{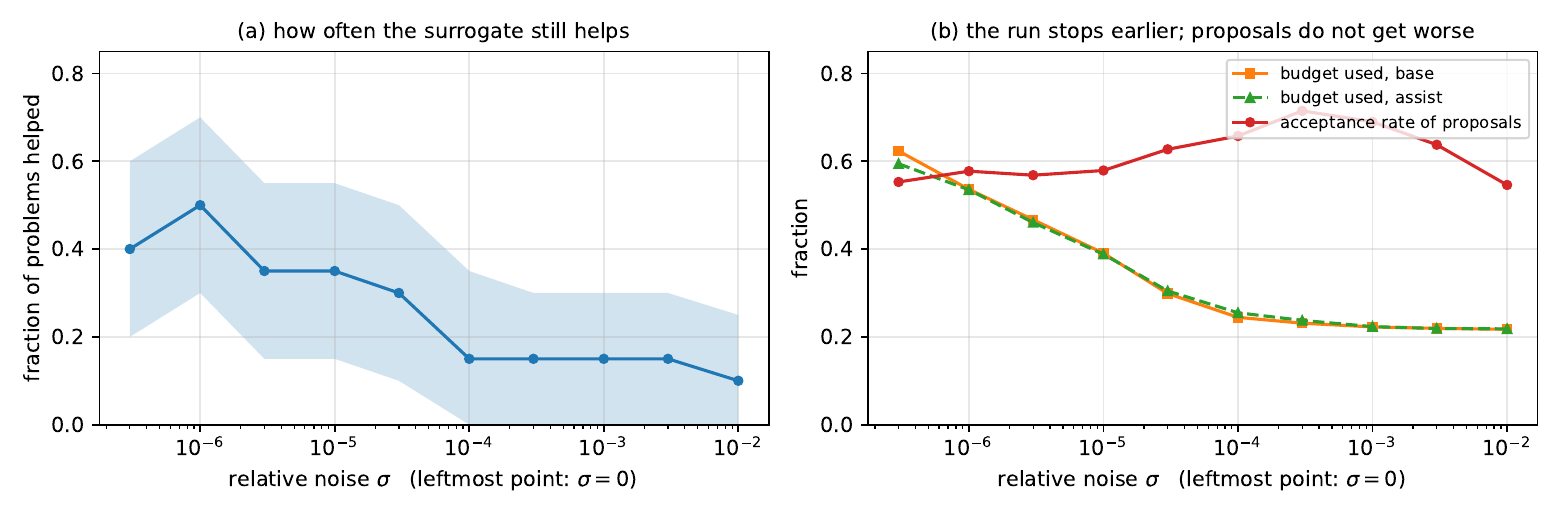}
\caption{Noise sweep, $200$ cells. (a) The fraction of problem--dimension pairs on which
safeguarded assistance still ends at least $1.26\times$ lower, with a bootstrap band over
problems; the decline is real but the band is wide throughout. (b) On the same runs, the
fraction of the evaluation budget consumed before the run stops, and the acceptance rate of
surrogate proposals. The runs get shorter; the proposals do not get worse.}
\label{fig:noise_sweep}
\end{figure}

\paragraph{There is no threshold in the noise.} Figure~\ref{fig:noise_sweep}(a) shows the
fraction of the twenty problem--dimension pairs that are helped falling from $0.70$ at
$\sigma=0$ and $0.75$ at $\sigma=10^{-6}$ to $0.35$ at $\sigma=10^{-4}$ and $0.15$ at
$\sigma=10^{-2}$. The bootstrap intervals overlap throughout, and the per-problem picture
explains why: the largest noise level at which a pair is still helped ranges over four
decades, from $10^{-6}$ on extended Freudenstein--Roth at $n=4$ to $10^{-2}$ on the
trigonometric function and the discrete integral equation, with a median of
$3\times10^{-5}$; three pairs are helped only in the noiseless case and one, variably
dimensioned at $n=8$, at no level at all. Where a pair is helped the effect is
substantial, the median gain among helped cells being $0.21$ in $\log_{10}$ with an
interquartile range of $[0.08,0.65]$, but which pairs those are depends on the problem far
more than on the noise. A clean cliff at $\sigma\approx3\times10^{-5}$ appears if one looks at two problems only; it
is the threshold of the Powell singular function, which the wider sweep reproduces exactly,
and not a property of the method.

\paragraph{What does change is the base method.} Figure~\ref{fig:noise_sweep}(b) locates the effect elsewhere. The fraction of the evaluation budget consumed before a run stops falls
monotonically from $62.3\%$ at $\sigma=0$ to $21.7\%$ at $\sigma=10^{-2}$, and the base
method and the assistance method track each other to within three percentage points at
every level: they stop at the same place. Surrogate proposals across the sweep fall from
$18682$ to $152$, a factor of $123$. Their acceptance rate falls too, from $0.843$ at
$\sigma=0$ to $0.520$ at $\sigma=10^{-2}$, but by a factor of $1.6$ against the proposals'
$123$. Two orders of magnitude separate the two effects, and only the smaller of them is
about the surrogate: proposals disappear because the run has already stopped, not because
they are being rejected.

The cause is the base method's Armijo test, which compares two noisy evaluations and, once
the noise exceeds the decrease being tested for, stops accepting anything; the inner loop
then shrinks the finite-difference stepsize indefinitely and the run stalls. Every surrogate proposal is validated by a test of exactly this kind, so a safeguard
inherits the noise tolerance of the test it relies on and cannot exceed it. Unsafeguarded
replacement is actively harmful: at $\sigma=10^{-6}$ on Powell it returns a median of $5.4$
against $0.066$ for the base method.

\paragraph{What would make a surrogate pay at this noise level?} None of it is a matter of
training the network better. The binding quantity is the variance entering the acceptance
test, which the surrogate inherits twice, through gradient targets that are finite
differences of noisy runs and through the test that validates its proposals. Gradient
targets would have to be estimated rather than differenced, by common random numbers,
batch regression, or an unbiased estimator where the model admits one; the acceptance test
would have to become a statistical comparison at a sample size the method controls; and the
surrogate would have to be trained on replication-averaged targets of known noise level.
Each is a change to the interface between model and oracle rather than to the model, which
is what the ablation below finds from the other side.

\paragraph{The third condition.} This qualifies the role and radius conditions rather than
replacing them. A surrogate accelerates a base method
that is still making progress and cannot accelerate one that has stopped, so the budget the
base method leaves unspent is the room available to any acceleration; under noise that room
closes well before the surrogate becomes the limiting factor. The consequence is an
ordering: make the acceptance test survive the noise first, through replication, common
random numbers or a statistical test in place of a deterministic sufficient-decrease
inequality, and only then ask what a surrogate adds. Section~\ref{sec:exp_sim} puts that
ordering to the test on a real simulation model.

\subsection{A Simulation Testbed, and What Breaks on It}\label{sec:exp_sim}

Section~\ref{sec:exp_noise} adds noise to analytic functions, which isolates a mechanism
but is not a simulation. The ordering it suggests --- repair the acceptance test first,
then ask about the surrogate --- is a claim about practice, and we test it here on a
genuine Monte-Carlo model.

\paragraph{The model.} Four products share a replenishment channel under a periodic-review
$(s,S)$ policy, with Poisson demand of mean $\lambda=(5,10,15,20)$: when the inventory
position of product $j$ falls to $s_j$ or below it is raised to $S_j$. A period in which
at least one product orders incurs a joint setup cost $K=60$, and each ordering product
incurs a minor setup cost of $8$, a unit purchase cost $c_j$, a holding cost $h_j$ and a
backorder penalty $p_j$. The joint setup couples the eight decision variables
$x=(s_1,d_1,\dots,s_4,d_4)$, where $S_j=s_j+\max(d_j,0)$, so the problem is not separable.
One oracle call runs one replication over $60$ periods and returns the realised average
cost; the single-replication standard deviation at the starting design is $3.5$ against a
mean of $222.3$, a relative noise of $1.6\%$. Reported costs are expected costs of the returned design, estimated separately under
common random numbers, so no solver gains from a lucky read. A random search under the same
$2000$-replication budget reaches $206.8$; we use it as a rough indication of what is
reachable, not as an optimum.

\paragraph{The methods of Sections~\ref{sec:exp_role}--\ref{sec:exp_noise} do not work
here.} The comparison uses $30$ seeds at each of three starting designs, costing $222.3$,
$239.3$ and $226.1$, for $90$ runs per method with a budget of $2000$ replications each.
A seed fixes the demand stream, so every method meets the same randomness at the same
start; comparisons are paired by construction and we report medians of paired differences
with bootstrap $95\%$ confidence intervals. With a deterministic oracle interface the
finite-difference base method returns its starting design on the median run, and so do both
surrogate variants: the median paired difference against the base method is $0.00$ for
safeguarded assistance and $+2.32$ for replacement, neither significant. The failure is
complete rather than marginal, and it does not originate in the surrogate.

The diagnosis follows from the stepsize rule $h_i=2\epsilon/(5\sqrt{n}\,2^{i}\sigma_k)$,
which ties the difference stepsize to the backtracking index. That coupling is right in a
deterministic setting, where shrinking $h$ with the step reduces truncation error, but
under a stochastic oracle the noise in a difference quotient scales as $h^{-1}$, so each
backtrack \emph{doubles} the gradient noise. At the starting design $h_0\approx0.14$ puts
the noise on each gradient component near $35$ while the true partial derivatives are of
order one, and the computed ``gradient'' is a random vector of norm about $100$. The
Armijo test then asks for a decrease of $\tfrac{1}{8}\alpha\|g\|^2\approx1250$ from a cost
of $222$, which is unattainable, so the method backtracks, worsens its own gradient
estimate, raises its own acceptance threshold and repeats until the inner loop is
exhausted. Raising the stepsize multiplier does not repair this, since the required decrease grows
with $\|g\|^2$ while $\epsilon$ controls the stepsize, the stationarity test and the
acceptance threshold at once. This is the mechanism of Section~\ref{sec:exp_noise} at a
noise level, $1.6\times10^{-2}$, far past where the synthetic sweep shows runs stalling.

\paragraph{A noise-aware interface.} We therefore modify the oracle interface, leaving the
surrogate machinery untouched: the difference stepsize is decoupled from the backtracking
index and set from an estimate $\hat\sigma_\varepsilon$ of the replication standard
deviation, balancing truncation against a noise term of order $\hat\sigma_\varepsilon/h$;
every evaluation entering a gradient averages $r$ replications, so gradient noise falls as
$r^{-1/2}$; and acceptance requires the observed decrease to exceed a standard Armijo term
and, optionally, $z$ standard errors of it,
\begin{equation}\label{eq:na_accept}
\hat f(x)-\hat f(x^{+})\ \ge\ c\,\alpha\|g\|^{2}\;+\;z\,\widehat{\mathrm{SE}},
\qquad
\widehat{\mathrm{SE}}=\hat\sigma_\varepsilon\sqrt{2/r},
\end{equation}
with $c=10^{-4}$, $z=1$ and $r=10$. Nothing here is new; the same ingredients underpin
adaptive-sampling trust-region methods for simulation optimisation, of which ASTRO-DF
\cite{ShashaaniHashemiPasupathy2018} is the reference example, and our point is precisely
that they are \emph{necessary} before a question about surrogates can even be posed.

\paragraph{Result.} Table~\ref{tab:inventory} reports the comparison.
Repairing the interface is worth a median paired
improvement of $10.40$ cost units, with bootstrap interval $[8.84,12.74]$, the repaired
method being ahead on $73$ of $90$ runs. Adding safeguarded surrogate assistance on top of
the repair is worth nothing measurable: the median paired difference is $0.00$ with
interval $[0.00,0.18]$, and the surrogate variant leads on only $33$ of $90$ runs.

A five-seed pilot had suggested a gain of about $0.7$ cost units; it does not survive at
$N=90$, and we report the correction. The repaired base method reaches $211.6$ against the
$206.8$ a random search of the same budget reaches, capturing roughly nine tenths of the achievable reduction on
its own, so little is left for any acceleration to claim. This is not a budget artefact:
raising the budget from $2000$ to $8000$ replications changes nothing, because both methods
stop when backtracking fails rather than when the budget runs out, and at budgets from
$200$ to $1200$ the paired difference never favours assistance either. The ordering the rest
of the paper is about survives in its weaker form: assistance does its base method no
systematic damage at either interface, its median paired difference being zero on the
deterministic interface and statistically indistinguishable from zero on the noise-aware
one, though individual runs go both ways, whereas replacement trails the base method on $85$
of $90$ runs.

\paragraph{Why the surrogate contributes nothing: it is almost never invoked.} A null
result of this kind is uninformative unless one knows what produced it, so we instrumented
the surrogate loop. Across thirty runs on the four-product model it issued $31$ proposals
in all, against $10650$ over the $117$ deterministic instances, and on thirteen of the
thirty runs it issued none. Five were accepted, a rate of $0.161$ against $0.703$ on the
deterministic side. The surrogate is not adding a small benefit our sample size cannot
resolve; it is barely being consulted, because the base method stops before it accumulates
the data a surrogate needs.

This is a stronger statement than a null effect, and it has two consequences. The first
concerns the mechanism. The surrogate's value targets are clean, being averages of
ten replications, but its gradient targets are finite-difference gradients formed from
those same noisy runs, and the diagnosis above shows how badly differencing amplifies
replication noise. Section~\ref{sec:exp_sobolev} established on the deterministic side
that acceptance depends on gradient information in the loss, the rate falling from $0.703$
to $0.148$ when it is removed; under simulation noise the gradient targets degrade further
and the rate falls to $0.161$. The three measurements line up: the surrogate loop needs
trustworthy gradient information, and simulation-grade noise degrades it.

The second concerns the safeguard. A surrogate that contributes nothing measurable cost the
method $31$ evaluations out of $60{,}000$ replications, one part in two thousand, which is what the role
distinction of Section~\ref{sec:background} predicts: under assistance a useless model
costs one evaluation per rejected proposal, whereas under replacement a bad model cost
entire sequences of failed backtracks ($735$ fallbacks across the benchmark). The inventory
model is a limiting case of the thesis rather than evidence against it.

\paragraph{Is the null result only a ceiling?} One explanation for the null result is
that the repaired base method already reaches $211.6$ against the $206.8$ of a
random search, some $2.3\%$ away, so a surrogate may simply have had nothing left to win.
We tested that explanation rather than assuming it. Enlarging the model to eight products,
$n=16$, leaves $8.3\%$ of the starting cost available, four times the headroom, and the
repaired base method there returns a median cost of $421.4$ against a random-search best of
$416.9$. Over $90$ paired runs the surrogate again adds nothing: the median difference is
$+0.00$, with bootstrap interval $[-0.25,+0.91]$, and assistance leads on $37$ of $90$
runs. The null result is therefore not an artefact of a saturated small model. Whatever
governs it is the noise, not the amount of room.

Taken with Sections~\ref{sec:exp_role} and \ref{sec:exp_noise}, a sequence emerges. On the
deterministic benchmark, where the base method spends its whole budget, the surrogate lifts
it from $67$ instances to $84$ while a released interpolation solver reaches $103$ without
one. Across the noise sweep the base method's runs shorten to a fifth of their budget and
the surrogate's contribution shortens with them. On two Monte-Carlo models the base method
does not move at all until its oracle interface is repaired, that repair is worth $10.40$
cost units, and the surrogate on top is worth $0.00$ on both. What moves along the sequence
is the per-step decrease relative to the noise in the test that accepts it, which suggests
an ordering of effort: fix the sampling and acceptance machinery first, and only then ask
whether a surrogate is worth its cost.

\begin{table}[t]
\centering
\caption{Joint-replenishment $(s,S)$ system: expected cost of the returned design over
$90$ runs ($30$ seeds $\times$ $3$ starting designs). Paired differences are taken against
the FD base method at the same seed and start, with bootstrap $95\%$ intervals for the
median; ``wins'' counts runs in which the method beats that base.}
\label{tab:inventory}
\begin{tabular}{l rr r@{\;}l c}
\toprule
Method & median & IQR & \multicolumn{2}{c}{paired diff.\ vs FD base} & wins\\
\midrule
FD base method            & 222.3 & $[219.3,226.1]$ & --- & & ---\\
NN assist (Sobolev)       & 222.3 & $[218.4,226.1]$ & $+0.00$ & $[+0.00,+0.00]$ & 21/90\\
NN replace                & 226.1 & $[222.3,239.3]$ & $+2.32$ & $[+0.00,+5.69]$ & 5/90\\
\midrule
Noise-aware base          & \textbf{211.6} & $[209.2,215.0]$ & $\mathbf{-10.40}$ & $\mathbf{[-12.74,-8.84]}$ & 73/90\\
Noise-aware $+$ NN assist & 212.2 & $[208.6,217.1]$ & $-10.47$ & $[-12.24,-8.28]$ & 71/90\\
\midrule
Adaptive-sampling TR       & 211.4 & $[210.6,212.4]$ & \multicolumn{2}{c}{---} & ---\\
\bottomrule
\end{tabular}
\end{table}

\paragraph{How far is this from the state of the art?} We implemented the mechanism that
defines the reference class, adaptive sampling: replications are increased until the
standard error falls below $\kappa\Delta^2$, with model, step and ratio test from the
trust-region machinery of Section~\ref{sec:exp_role}. This is the idea behind ASTRO-DF
\cite{ShashaaniHashemiPasupathy2018}; that solver is not installable here, so we run a
reproduction of its sampling rule rather than the published code. Over the same $90$ runs it
returns a median of $211.4$, interquartile range $[210.6,212.4]$, against $211.6$ for our
noise-aware base method: indistinguishable at this sample size. The method the surrogate
fails to improve is therefore level with an adaptive-sampling reference rather than a
strawman, and all three sit within five cost units of the $206.8$ a random search attains.

\paragraph{Which part of the repair does the work?} The interface changes three things at
once, so we ablate them. The result is not the one we expected, and only one of the three components survives the
larger sample. Replication is decisive: weakening it to a single replication per
evaluation costs $12.72$ cost units, with interval $[9.49,15.07]$, and even three
replications cost $4.76$, with interval $[1.84,6.81]$. Neither of the other two changes
has a detectable effect. Recoupling the difference stepsize to the backtracking index
costs $0.29$ with interval $[0.00,1.28]$, and removing the standard-error guard of
\eqref{eq:na_accept} is likewise indistinguishable from keeping it. We had expected that
guard to be the essential ingredient. It is not.

The surviving reading is therefore sharper and narrower than the one we started with. What
binds is the \emph{variance of the estimates entering the test}, not the form of the test
and not the stepsize schedule: ten replications shrink that variance enough for an ordinary
Armijo comparison to become reliable, while a more careful rule layered on noisy estimates
buys nothing. This matches the design of adaptive-sampling trust-region methods
\cite{ShashaaniHashemiPasupathy2018}, where the algorithm controls the sample size rather
than the acceptance inequality, and it revises Section~\ref{sec:exp_noise}: a safeguard is
limited by its acceptance test, but the way to lift that limit is to reduce the noise in
what the test compares, not to make the test cleverer.

\paragraph{Caveats.} We use the simplified base method of \cite{Taminiau2025Arxiv2502}
with $B_k=0$, not the quasi-Newton version. Practitioners of simulation optimisation would
not apply single-replication differencing in the first place, so the negative result
should be read as a statement about transplanting a published deterministic method onto a
stochastic oracle unchanged, not as a claim that the method class is unusable. The
noise-aware variant is a minimal repair rather than a competitive solver; a proper
comparison against adaptive-sampling methods such as ASTRO-DF is future work.

\subsection{Do the Redundancy Diagnostics Predict Anything?}\label{sec:exp_diag}

Section~\ref{sec:tradeoffs} defines two diagnostics, the overlap ratio $r_k$ and the
step-spacing statistic $N_{\mathrm{sp}}$. Both are natural candidates for deciding when a
surrogate should be trusted, and neither had been validated against data. We do that here, on the same $117$
instances as Section~\ref{sec:exp_role} rather than on a subset of them: at every outer
iteration of the assistance solver we log the two diagnostics, the numerical rank of the
normalized displacements of the training inputs, and four further candidates computed from
quantities the method already holds, and correlate each with the fraction of surrogate
proposals accepted at that iteration. This gives $3202$ outer iterations across the $116$
instances on which the solver completes at least one; the remaining instance exhausts its
budget inside the first finite-difference gradient.

Two methodological choices affect the result. Outer iterations are clustered within
instances, so intervals computed as though the $3202$ rows were independent are far too
narrow; every interval in Table~\ref{tab:signals} comes from a bootstrap that resamples
\emph{instances},
which is the level at which the design varies. And several candidates are heavily tied:
$N_{\mathrm{sp}}$ equals $1$ at $77\%$ of all iterations. Rank correlations must
therefore use mid-ranks, since ordinal ranking breaks ties by position in the log and so
manufactures correlation out of the order in which iterations were written. Ordinal ranking makes $N_{\mathrm{sp}}$ look predictive for precisely that reason.

\begin{table}[t]
\centering
\caption{Candidate gate statistics against the surrogate acceptance rate: $3202$ outer
iterations on $116$ instances at the width-$6n$, $900$-iteration setting. Spearman $\rho$ with mid-ranks; intervals from $800$
bootstrap replicates resampling instances rather than iterations.}
\label{tab:signals}
\begin{tabular}{P{4.9cm} P{4.0cm} r c c}
\toprule
\textbf{Statistic} & \textbf{What it measures} & $\rho$ & \textbf{95\% CI} & \textbf{Usable}\\
\midrule
$\kappa_k=\|\grad m_k\|/\|g_k^{\mathrm{FD}}\|$ & gradient-scale disagreement & $-0.492$ & $[-0.532,-0.439]$ & yes\\
acceptance rate at $k-1$ & recent track record & $+0.465$ & $[+0.392,+0.522]$ & yes\\
$\varrho_k$, training residual & fit quality on own data & $-0.394$ & $[-0.467,-0.307]$ & yes\\
$\cos(\grad m_k,g_k^{\mathrm{FD}})$ & gradient-direction agreement & $+0.186$ & $[+0.132,+0.232]$ & weak\\
decrease ratio at $k-1$ & overshoot of the last failed proposal & $-0.155$ & $[-0.234,-0.065]$ & mislabelled\\
$r_k$, overlap ratio & across-iteration redundancy & $-0.097$ & $[-0.176,-0.020]$ & negligible\\
$N_{\mathrm{sp}}$, step spacing & within-line-search redundancy & $-0.063$ & $[-0.096,-0.032]$ & negligible\\
directional rank & collapse onto a subspace & $-0.029$ & $[-0.145,+0.087]$ & no\\
\bottomrule
\end{tabular}
\end{table}

The results separate the coverage diagnostics from the model-quality ones, and the
separation is one of size rather than of existence. Two of the three coverage diagnostics
are statistically distinguishable from zero at this sample size, $-0.097$ for $r_k$ and
$-0.063$ for $N_{\mathrm{sp}}$, and the sign is the intuitive one: more redundancy, slightly
worse odds. But the effects are five to eight times smaller than those of the model-quality
statistics, $-0.492$ for the gradient-scale ratio and $+0.465$ for the previous acceptance
rate, and a diagnostic that moves the acceptance odds by a hundredth is not a gate. The
directional rank is not distinguishable from zero at all, and for it there is a plain
explanation: it is nearly always full, averaging $3.83$ of $4$, $7.54$ of $8$ and $14.42$ of
$16$ and attaining the maximum at $92\%$, $90\%$ and $81\%$ of iterations respectively, so
there is little variation for it to explain. Whatever limits the surrogate here, it is not a
training set collapsing onto a subspace.

The model-quality diagnostics do work. The strongest is the ratio
$\kappa_k=\|\grad m_k\|/\|g_k^{\mathrm{FD}}\|$ between the surrogate and finite-difference
gradient norms at $\rho=-0.492$, then the acceptance rate at the previous outer iteration
at $+0.465$ and the training residual $\varrho_k$ at $-0.394$. Each carries the sign a
mechanism would predict, an over-confident or poorly fitted model being the one whose
proposals fail, and the effect sizes are usable rather than marginal. Sorting outer
iterations by $\kappa_k$, the acceptance rate falls monotonically across quartiles from
$0.484$ to $0.343$, $0.106$ and $0.045$; split at $\kappa_k=2$ it is $0.415$ below against
$0.076$ above, a gap of $0.339$ with interval $[0.292,0.384]$. The previous-iteration
acceptance rate separates comparably, $0.475$ against $0.133$. One entry in Table~\ref{tab:signals} looked anomalous and turned out to be mislabelled, so
we record what it actually measures. The decrease ratio at $k-1$ correlates $-0.155$ with
interval $[-0.234,-0.065]$, the opposite of the sign a trust-region argument predicts. The
explanation is that the quantity logged is not the trust-region ratio of the previous
\emph{accepted} step but that of the previous \emph{proposal}, and since the surrogate loop
exits on its first rejection, that proposal is the rejected one at $99.2\%$ of iterations;
its median value is $-0.202$ and it exceeds $1$ at eleven of $3202$ iterations. Read correctly it is a measure of
how badly the last failed proposal overshot, and a worse overshoot at $k-1$ predicting a
worse model at $k$ is the expected sign, not a reversed one. Recomputing the statistic over
accepted steps only gives $+0.309$, the sign theory predicts, on the $1053$ of $3202$
iterations where the previous loop accepted anything at all. That last figure is the reason
we do not gate on it: a trust-region ratio is undefined at two thirds of the iterations
where a gate would have to fire, whereas $\kappa_k$ and $\varrho_k$ are defined at all of
them.

Two consequences follow. The concern raised about $N_{\mathrm{sp}}$ in
Section~\ref{sec:tradeoffs} was understated: it measures no directional coverage and what it
does predict is an order of magnitude too weak to act on, and the same holds for $r_k$. And
since the framework of Section~\ref{sec:aras} needs \emph{some} gate, $\kappa_k$ and
$\varrho_k$ supply one on the strength of an effect five to eight times larger.

\paragraph{Calibrating the gate.} A predictive signal is not yet a decision rule, so we
calibrate one. Gating the loop off at an outer iteration saves the $p$ evaluations its
proposals would have cost and gives up the $a$ descent steps it would have supplied, each
of which the base method must otherwise buy with a finite-difference gradient costing
$n+1$ evaluations; the net saving is $p-a(n+1)$. Maximising that quantity over $\kappa_k$
thresholds on half the instances, chosen at random, gives the same answer at every
candidate value: it is negative throughout, and the optimum is not to gate at all. Gating
pays only where the acceptance rate falls below $1/(n+1)$, that is $0.200$, $0.111$ and
$0.059$ at $n=4,8,16$, and even in the least trustworthy $\kappa_k$ quartile the observed
rates are $0.355$, $0.382$ and $0.114$. Two conventions differ from those of
Section~\ref{sec:aras} and both matter. These are pooled ratios, total accepted over total
proposed, which is what the accounting needs, whereas the quartile series quoted there is a
mean of per-iteration ratios and is lower because $90\%$ of high-$\kappa_k$ iterations make
a single proposal; and the quartiles here are taken within each dimension, since the
break-even $1/(n+1)$ is a per-dimension quantity. Every one clears its break-even, though at $n=16$ by a factor of
two rather than by the wide margins of the two smaller dimensions.

\paragraph{Validating it.} Running the gated solver at four thresholds against ungated
assistance on all $117$ instances confirms the prediction. The validation set contains the
calibration half, which we report rather than conceal: the calibrated answer is a corner
solution, not a fitted threshold, and what is being validated is a fresh solver run rather
than a re-scoring of the logs the threshold was chosen on. Ungated assistance solves $109$
instances at $\tau=10^{-5}$; gating at $\kappa_k>1$, $2$, $5$ and $30$, which fires at
$71.5\%$, $52.5\%$, $41.6\%$ and $28.9\%$ of iterations, solves $97$, $103$, $108$ and
$109$. The aggressive settings lose outright, the permissive ones converge back to the
ungated method as they stop firing, and the paired evaluation difference on jointly solved
instances is zero throughout. Gating at $\kappa_k>2$ avoids $1107$ proposals but forfeits
$925$ accepted steps, a net loss of about $8400$ evaluations at the mean $n+1\approx10.3$
here. Training the network properly widens the loss rather than narrowing it: the same
comparison at the undertrained setting cost six instances at $\kappa_k>1$ where it now costs
twelve, because a gate that fires on a model worth trusting forfeits more.

This is a completed calibration with a negative outcome rather than an open question, and
the outcome is informative: safeguarded assistance is already cheap enough that no gate on
top of it pays for itself while evaluations are cheap and the model is merely imperfect.
The criterion also says where a gate would pay, namely once acceptance drops below
$1/(n+1)$. The inventory model of Section~\ref{sec:exp_sim} does not supply the counter-example we
expected: with $n=8$ its break-even is $0.111$ and the measured rate is $5/31=0.161$, above
it, so the same rule leaves the gate off there too. What that model illustrates is the third
condition rather than a use for the gate\,--\,a surrogate consulted $31$ times in $60{,}000$
replications is not something a gate can improve. The correlations remain moderate, so $\kappa_k$ ranks iterations usefully rather than
separating them cleanly; what we can now state is the gate's operating regime, and that
this benchmark lies outside it.

\subsection{Is the Network the Problem?}\label{sec:exp_arch}

The architecture used throughout is not a free choice, and this section is where it was
made. We varied the two settings controlling capacity and convergence on all thirteen
problems at $n=8$: width in $\{n,3n,6n,12n\}$ and training length in $\{300,900\}$. An
earlier version of this paper reported every surrogate result at width $3n$ and $300$
iterations; the ablation showed that setting to be undertrained, and the whole benchmark,
the loss ablation, the noise sweep and the high-dimensional runs were repeated at the
setting selected here. Depth is varied separately, in Section~\ref{sec:exp_region}, and
value-only in both arms, since the analytic parameter gradients of the Sobolev loss in
Section~\ref{sec:exp_setup} are derived for a single hidden layer.

Training length matters and width barely does. At $300$ iterations every width solves $6$ of
$13$ at $\tau=10^{-5}$, exactly matching the base method; at $900$ the same widths solve
$8$, $9$ and $9$ for $3n$, $6n$ and $12n$, and the two best configurations, widths $6n$ and
$12n$, are indistinguishable. Capacity is not the binding constraint\,--\,width $n$ to $12n$
changes nothing at $300$ iterations and little at $900$\,--\,and convergence is. We take
$6n$ and $900$ forward.

That correction is the largest in this paper: the same embedding solves $75$ instances at
$\tau=10^{-5}$ at width $3n$ and $300$ iterations against $84$ at the setting above, so a
third of the measured gap between assistance and a model-based solver was a property of the
fit. What survives is every factor above, at the trained setting: replacement still loses to
its base, $65$ against $67$; value-only training leaves the surrogate below it at $66$; the
trust-region method still gains nothing, $86$ against $88$; released interpolation software
is ahead at every dimension. Better training moved the level and left the ordering. We
varied three settings one at a time and did not search the architecture jointly.

\subsection{What the Surrogate Costs in Time}\label{sec:exp_time}

Everything above counts objective evaluations, the usual currency in derivative-free
optimisation because in the applications that motivate it one evaluation dominates
everything else. That accounting hides the surrogate's own cost, which is not negligible for cheap oracles.
We measure it directly.

One training call, at the dataset caps of Section~\ref{sec:exp_setup} and with the $300$
Adam iterations used throughout, takes a median of $43$\,ms at $n=4$, $57$\,ms at $n=8$ and
$111$\,ms at $n=16$ in our pure-NumPy implementation. Evaluating the trained surrogate's
gradient costs about $18\,\mu$s and is negligible by comparison; training dominates. The
analytic test objectives, at $5.5\,\mu$s per call, are some four orders of magnitude
cheaper than a training call, so on the benchmark of Section~\ref{sec:exp_role} the
surrogate variants are far slower in wall-clock time than the base method even where they
are far cheaper in evaluations. This is expected rather than a defect of the method: it is precisely why the evaluation
count is the appropriate currency only for expensive oracles.

It remains to quantify how expensive an oracle must be. Over the $117$ instances the
assistance solver performs $3211$ trainings, issues $10650$ proposals and has $7488$ of them
accepted, that is $3.32$ proposals and $2.33$ accepted steps per training. An accepted
surrogate step is a descent step obtained without a finite-difference gradient, so it
displaces the $n+1$ evaluations that gradient would have cost, while every proposal costs
one true evaluation whether or not it is accepted. The net saving per training is therefore
about $2.33(n+1)-3.32$ evaluations, namely $8.3$ at $n=4$, $17.7$ at $n=8$ and $36.3$ at
$n=16$, and the surrogate pays for its own time once a single evaluation costs more than
$16.7$\,ms at $n=4$, $12.4$\,ms at $n=8$ and $14.7$\,ms at $n=16$. Crediting an accepted
step with a whole finite-difference gradient is generous, so these are lower bounds on the
oracle cost at which assistance becomes worthwhile. Training the network properly raises
both sides of that ledger: the yield per training rises by a factor of two and a half and
the cost of one training by a factor of three to five, so the threshold barely moves at
$n=4$ and rises at $n=8$ and $n=16$;
Section~\ref{sec:discussion} shows it continuing to rise beyond $n=16$.

Ten milliseconds is a low threshold for the simulation models that motivate this work,
since a single replication of the inventory model of Section~\ref{sec:exp_sim} already
exceeds it at the replication counts used there, but a high one for an analytic benchmark
function. The wall-clock case for a surrogate is therefore governed by the cost of the
oracle rather than by the surrogate itself. An optimised training implementation would
lower the threshold roughly in proportion to the speed-up it achieves.

\subsection{Scope and Limitations}\label{sec:exp_limits}

One limitation deserves to be stated before the others, because it bounds every claim here.
Our main studies use $n\le16$ and the probe of Section~\ref{sec:exp_role} reaches $n=128$ on
thirteen problems with one start each. Neural surrogates interest the field largely because
they are expected to scale where interpolation models do not, and at the dimensions we
reach that expectation is not met: a released interpolation solver is ahead at every one,
and the margin at $n=128$ is the margin at $n=4$. What we cannot claim is that this
continues. The thirteen problems are smooth and unconstrained, one start each at the two
largest dimensions, and $n$ in the thousands, where interpolation's $O(n^2)$ model
bookkeeping is genuinely prohibitive and subspace or sketching methods take over, is out of
reach here: a single $n=128$ run of the released solver took us up to $33$ hours. Whether
the ordering we observe persists there is the most important thing this paper leaves open,
and our own experience at $n=64$ is a caution against settling it with a
reimplementation. The experiments are limited in the usual ways as well. The deterministic study
uses one run per instance, which is exact because the solvers are deterministic there; the
region-size sweep covers $273$ cells and the noise sweep $200$, each at five seeds; the
inventory studies use $30$ seeds at three starting designs on two model sizes. The
diagnostic study of Section~\ref{sec:exp_diag} pools outer iterations across problems, so
its correlations still mix within-run and between-problem variation even though its
intervals are computed at the instance level. A word on the reference methods. Both are our implementations of the published schemes,
so we checked the one that can be checked. Running released Py-BOBYQA \cite{Powell2006NEWUOA} on all $117$ benchmark instances under
the same budget and the same counting oracle, the median $\log_{10}$ ratio of final
objective values between our trust-region method and the released solver is $0.000$, with
ours ahead on $38$ of $117$; neither dominates and the typical difference is nil. That agreement holds at $n\le16$ only. At $n=64$ it fails badly, our implementation solving
$5$ of $13$ against $12$ for the released solver, and Section~\ref{sec:exp_role} reports
what that cost us: a crossover we had measured and believed turned out to be an artefact of
our own code. A reimplementation validated at one problem size is not validated at another,
and the dimensions where an interpolation method is under most pressure are exactly where
careful engineering separates it from a straightforward version. The adaptive-sampling method of Section~\ref{sec:exp_sim} has no such check, no released
ASTRO-DF being available to us, and its numbers are a reference point rather than a
benchmark; released implementations do exist in the simulation-optimisation testbed
literature and running against one is the obvious next step there. Three things would settle what this paper leaves open. Our $n=128$ runs are incomplete, and
since the argument for learned surrogates concerns the regime where a full quadratic is
impractical, $n$ in the hundreds is where the question is decided. A formal benchmark on the
CUTEst and Mor\'e--Wild suites, noisy variants included, should use released code
throughout; our experience at $n=64$ is why we call that necessary rather than desirable.
And the assistance mechanism itself should be mounted on released software rather than on a
base method we wrote: the role distinction predicts that a safeguarded loop attached to
Py-BOBYQA, or to a released adaptive-sampling solver in the stochastic case, would cost
little and gain little, and that prediction is worth testing where a failure would be
visible. All of this quantifies mechanisms rather than state-of-the-art
performance, and the wall-clock figures of Section~\ref{sec:exp_time} come from an
unoptimised pure-NumPy trainer and so bound the surrogate's overhead from above.

\section{Discussion}\label{sec:discussion}
Our synthesis yields practical and theoretical implications.

\paragraph{What the experiments support.}
The measurements in Section~\ref{sec:experiments} line up with the mechanisms we have argued for. Holding the surrogate class and the training pipeline fixed, safeguarded assistance improved evaluation efficiency and gradient replacement degraded it. Uniform error grew quickly with the sampling radius and failed to shrink in proportion at small radii. Removing the gradient term from the loss cut the acceptance rate from $0.703$ to $0.148$. The stochastic studies add a fourth observation, and it is the one a simulation reader
should weigh most heavily. What noise removes is not the surrogate's accuracy but the base
method's ability to keep going: across the sweep its runs shorten until they use a fifth of
their budget, and the surrogate's contribution shortens with them: proposals fall by a
factor of $123$ while the rate at which they are accepted falls by a factor of $1.6$. On an actual Monte-Carlo model every method with a
deterministic oracle interface stalls at its starting design; replication repairs the
interface and recovers most of the achievable improvement, after which the surrogate adds
nothing. Over $90$ paired runs the median difference is zero, it stays zero on a larger
instance built to leave four times as much room, and the loop is invoked $31$ times in
$60{,}000$ replications against $10650$ over the benchmark. The safeguard holds the cost of
that idle model to one part in two thousand of the budget. Assistance does no systematic damage, which
is weaker than saying it does none: its median paired difference against the base method is
zero, but run by run it is worse on $20$ of the $90$ and better on $21$, the rest tied.
Replacement, by contrast, trails on $85$ of $90$. Read with the benchmark, where the same surrogate lifts high-accuracy solutions from $67$
to $84$ of $117$ over its own base, stays behind the $103$ a released interpolation solver
reaches alone, and lowers our own such solver from $88$ to $86$ when attached to it, and
where that ordering holds at $n=32$, $n=64$ and $n=128$ as well, the results bracket the regime in which
learned surrogates earn their cost: they close gaps that the base method leaves open, and
only those. On smooth problems the gap they close is worth as much as a competent
interpolation model at $n\le16$ and less than a well-engineered one above it. Our
experiments therefore establish the mechanism by which a learned local model helps, and
locate its usefulness where the base method is weak and the dimension is low enough that a
derivative-free budget can still buy the data the model needs.

\paragraph{What the mechanism faces above $n=100$.} Our experiments stop at $n=128$, so
the regime the motivation is really about is out of reach, but three of the surrogate's own
costs can be measured there without running a solver, and they point the wrong way. Training
one surrogate at the dataset caps used throughout takes $140$\,ms at $n=4$ and $31$\,s at
$n=128$, growing as $n^{2.0}$ over the upper half of that range, because the parameter count
at width $6n$ grows as $6n^{2}$. Each finite-difference gradient target costs $n+1$
evaluations, so the ten targets the training set holds consume a fifth of the evaluation
budget at every dimension, $1290$ of $6450$ at $n=128$. Combining these with the yield
measured in Section~\ref{sec:exp_time}, the oracle cost at which assistance pays for its own
wall-clock time is $16.7$, $12.4$ and $14.7$\,ms at $n=4$, $8$ and $16$, then $25.6$, $56.5$
and $104.4$\,ms at $n=32$, $64$ and $128$. The surrogate's economics improve with dimension
only while the avoided finite-difference gradient grows faster than the cost of fitting the
model, and past $n\approx8$ they stop doing so.

The arithmetic behind that turn is worth setting out, because it also explains why the
training correction of Section~\ref{sec:exp_arch} does not carry to high dimension. The
number of training points is fixed by the evaluation budget at $10(n+1)$, since each
gradient target costs a finite-difference gradient; the parameter count at width $6n$ is
$6n^2+12n+1$. The ratio therefore grows linearly, from $5$ parameters per point at $n=8$ to
$39$ at $n=64$ and $77$ at $n=128$. Widening the network and training it longer cures
underfitting at $n\le16$, where the ratio is single-digit, and cannot help where the ratio
is already large: the tuned setting is worth $18$ instances of $117$ at $n\le16$ and two of
thirteen at $n=128$. What limits a learned local model in high dimension on this evidence is
not its expressive power but the number of points a derivative-free budget can buy.

This is the opposite of the expectation that motivates learned surrogates. A network's
advantage over interpolation is supposed to be that its parameter count need not track the
number of interpolation conditions; ours grows quadratically because we scale width with
$n$, and a fixed-width network would trade that against capacity. The measurements indict
our scaling choice, not the idea, which makes an architecture whose cost grows slowly
without its local gradient accuracy collapsing the substantive open problem here, rather
than the same experiment at larger $n$.

The comparison method is not cheap at $n=128$ either: the thirteen released-solver runs of
Section~\ref{sec:exp_role} took $90$ hours between them, from $16$ seconds on the sphere to
$33$ hours on extended Wood, on an objective whose own evaluation costs microseconds.
Maintaining a $257$-point model for $6450$ evaluations is itself expensive. Both figures are
properties of Python implementations rather than of the two approaches, so we draw no
ranking from them; the point is only that at this dimension the wall-clock argument does not
separate the methods the way the evaluation counts do.

\paragraph{Smoothness vs trend prediction.}
Smoothing reduces oscillations induced by noisy or poorly tuned finite differences; trend
prediction produces descent directions that correlate with the objective's decrease over a
neighbourhood. Sobolev training encodes the latter and, under finite-difference stencils,
penalises surrogate curvature along coordinate directions
(Proposition~\ref{prop:sobolev_curvature}) \cite{Taminiau2025Arxiv2502}. With weight decay
and dataset capping this stabilises local surrogate descent in a way raw finite differences
do not.

\paragraph{Why ``replace'' can fail even with good approximation.}
Giovannelli et al.\ find that surrogate gradients do not reliably improve a strong
FD-BFGS-based method \cite{Giovannelli23T027}, which is consistent with the sensitivity of
quasi-Newton updates: inexact gradients affect both the descent direction and the curvature
update, and the errors accumulate. Without uniform error bounds or acceptance safeguards
the algorithm over-trusts the surrogate.

\paragraph{Why ``assist'' can succeed.}
Assistance designs use the surrogate to propose steps but validate each step by sufficient decrease in the true objective (Algorithm~\ref{alg:surrogate}). This decouples ``surrogate imagination'' from ``true objective acceptance'' and controls failure modes. Taminiau et al.\ formalize how repeated successful surrogate steps improve evaluation-complexity constants via the gain factor $\eta(S)$ \cite{Taminiau2025Arxiv2502}. Our propositions (Section~\ref{sec:errorbounds}) further clarify that, under bounded gradient error, surrogate steps are more likely to align with true descent, thereby passing the true-decrease filter.

\paragraph{Generalization as a radius question.}
Both papers implicitly highlight that surrogates are local: Giovannelli et al.\ use ball sampling and observe radius effects; Taminiau et al.\ cap datasets and operate within stepsize-controlled neighborhoods. Our effective generalization radius (Definition~\ref{def:rgen}) reframes generalization as ``how far can we trust the surrogate?'' This enables direct experimental estimation by radius sweeps and provides guidance on region-size control in learning-augmented loops.

\paragraph{Speed and cost drawbacks.}
NN surrogates impose training overhead and hyperparameter sensitivity. Giovannelli et al.\ emphasize that NNs can be evaluation-competitive at high training cost \cite{Giovannelli23T027}. Taminiau et al.\ restrict model complexity (shallow NN) and use dataset caps and warm starts to limit overhead \cite{Taminiau2025Arxiv2502}, but inference about the best surrogate (NN vs RBF) remains problem dependent. Our proposed experiments emphasize measuring both evaluation efficiency and wall-clock time.

\section{Conclusion}\label{sec:conclusion}
The two questions of the title have answers, and they are not the same answer. A neural
surrogate improves local approximation because it fits at once every evaluation the run has
already paid for, which averages the noise that differencing amplifies, and because
gradient information in the training loss restrains the curvature a value-only fit would
invent: remove that term and the acceptance rate falls from $0.703$ to $0.148$, taking the
instances solved to high accuracy back to exactly the base method's count. It improves the
\emph{optimisation} only under conditions we delimit empirically rather than prove. Three
factors bound the gain, on evidence of three different kinds. \emph{Role} we fail on purpose
and watch the benefit go: moving the surrogate from proposing candidates to supplying the
gradient drops it from $84$ instances to $65$, below the $67$ of the base method it was meant
to help. \emph{Radius} we establish at the level of approximation rather than of the solver;
that the same threshold binds inside a running solver follows from
Proposition~\ref{prop:descent_inexact}, not from a separate experiment. And what an earlier
draft called \emph{room} is two things we now keep apart: \emph{headroom}, the performance
the base method leaves on the table, which a model-based solver has little of; and
\emph{viability}, whether its acceptance test still functions, which noise destroys. The
eight-product inventory instance separates them, having four times the headroom and still no
gain, so headroom is not sufficient once viability has gone. Reading Giovannelli et al.\ and Taminiau et al.\ through the first
of them, algorithmic role, accounts for their opposite conclusions without either being
wrong: replacement is fragile because surrogate error enters the core updates, whereas
safeguarded assistance preserves robustness and can improve the evaluation-complexity
constant through the surrogate-gain mechanism.

We proposed a radius-aware notion of local generalisation with diagnostics for the quality
of the local data, and tested it. Assistance improved evaluation efficiency where
replacement degraded it; surrogate reliability fell away quickly outside a bounded region
and did not improve without limit inside it; Sobolev training was what made surrogate steps
survive validation against the true objective. Under a stochastic oracle no noise threshold
separates the regimes: the largest level at which a problem is still helped ranges over four
decades, and what shortens with noise is the base method's run, not the surrogate's success
rate.

That last finding is the one we would take forward. Safeguarding keeps a learned model from damaging a solver, but it offers no protection that the underlying acceptance test does not already provide, and on a real simulation model that limit binds immediately: with a deterministic oracle interface, every variant we tried returned its starting design. The minimal repair of Section~\ref{sec:exp_sim} removes the obstacle and restores the ordering observed in the deterministic experiments. Its ablation corrects our initial guess twice over: what matters is replication, not the
shape of the acceptance rule and not the stepsize schedule. The practical rule we would
draw from the three studies together is an ordering of effort. Repair the sampling and
acceptance machinery first, because nothing else matters until the acceptance test can tell
a real decrease from noise. Then ask whether a surrogate pays, and expect the answer to turn on how much headroom the
base method leaves and on whether its acceptance test still functions: a great deal of room
on the benchmark, essentially none on the inventory model. Role, radius and room are cheap
to check before any network is trained, and where they hold the improvement is real and
reproducible; where they do not, no amount of approximation accuracy substitutes for them.
That is why a question about \emph{when} admits a sharper answer than a question about
\emph{whether}.

What we have established is a set of base mechanisms and the conditions under which each
operates, measured where every variant can be run on every instance. Three tasks follow.
The dimensions at which interpolation genuinely breaks down, $n$ in the thousands, are
where Section~\ref{sec:discussion} predicts the picture should change; our evidence says
the obstacle there is the number of points a derivative-free budget can buy rather than the
network's capacity, which points at architectures whose parameter count does not track $n$.
Extended Wood shows the set of objectives on which a learned local model beats released
interpolation software is not empty, and a diagnostic identifying that set in advance would
turn an existence result into a rule. And the ordering of effort above should be tested
against a released adaptive-sampling solver \cite{ShashaaniHashemiPasupathy2018} rather than
the reproduction of its sampling rule we ran, on the CUTEst and Mor\'e--Wild suites.

\end{document}